\pdfoutput=1
\documentclass{article}

\usepackage{blindtext}

\usepackage[preprint]{jmlr2e}
\usepackage{times}
\usepackage{bm}
\usepackage{booktabs}
\usepackage{subcaption}
\usepackage{graphics}
\usepackage{algorithm}
\usepackage{algorithmic}
\usepackage{wrapfig}
\usepackage{amsmath}
\hypersetup{
    colorlinks=true,
    citecolor=blue,
    linkcolor=magenta,
    urlcolor=blue
}

\makeatletter
\providecommand{\editor}[1]{}   
\makeatother

\newtheorem{thm}{Theorem}[section]
\newtheorem{prop}{Proposition}[section]
\newtheorem{lem}{Lemma}[section]

\newtheorem{rem}{Remark}[section]

\newtheorem{dfn}{Definition}[section]

\newtheorem{assum}{Assumption}[section]

\newcommand{\UPDATE}[1]{{#1}}

\usepackage{lastpage}

\firstpageno{1}

\begin{document}

\title{Using Stochastic Gradient Descent to Smooth Nonconvex Functions: Analysis of Implicit Graduated Optimization}

\author{\name Naoki Sato \email naoki310303@gmail.com \\
       Meiji University\\
       \AND
       \name Hideaki Iiduka \email iiduka@cs.meiji.ac.jp \\
       Meiji University}

\maketitle

\begin{abstract}
The graduated optimization approach is a method for finding global optimal solutions for nonconvex functions by using a function smoothing operation with stochastic noise. \UPDATE{This paper makes three contributions regarding graduated optimization. First, we extend the definition of function smoothing that is traditionally achieved through convolution with Gaussian noise and characterize for the first time function smoothing with heavy-tailed noise. Second, we show that light- or heavy-tailed stochastic noise in stochastic gradient descent (SGD) has the effect of smoothing the objective function, the degree of which is determined by the learning rate, batch size, and the moment of the stochastic noise. Using this finding, we propose and analyze a new graduated optimization algorithm that varies the degree of smoothing by varying the learning rate and batch size. Third, we relax the $\sigma$-nice property, a standard but restrictive condition in the analysis of graduated optimization. Our refinement enables convergence guarantees for a broader class of non-convex functions, thereby bridging the gap between theoretical assumptions and practical optimization landscapes.}
\end{abstract}


\section{Introduction}
\label{sec:intro}
\subsection{Background}
\label{sec:1.1}
The amazing success of deep neural networks (DNN) in recent years has been based on optimization by stochastic gradient descent (SGD) \citep{Her1951Ast} and its variants, such as Adam \citep{Die2015Ame}. These methods have been widely studied for their convergence \citep{Moulines2011Non, Needell2014Sto, Benjamin2020Con, Bottou2018Opt, Kevin2020Rob, Nicolas2021Sto, Zaheer2018Ada, Fangyu2019ASu, Xian2019Ont, Dongruo2020Ont, Iiduka2022App} and stability \citep{Moritz2016Tra, Jun2016Gen, Wen2018Gen, Fen2019Con} in nonconvex optimization.

SGD updates the parameters as $\bm{x}_{t+1}:=\bm{x}_t - \eta \nabla f_{\mathcal{S}_t} (\bm{x}_t)$, where $\eta$ is the learning rate and $\nabla f_{\mathcal{S}_t}$ is the stochastic gradient estimated from the full gradient $\nabla f$ using a mini-batch $\mathcal{S}_t$. Therefore, there is only an $\bm{\omega}_t := \nabla f_{\mathcal{S}_t} (\bm{x}_t) - \nabla f(\bm{x}_t)$ difference between the search direction of SGD and the true steepest descent direction. Some studies claim that it is crucial in nonconvex optimization. For example, it has been proven that noise helps the algorithm to escape local minima \citep{Rong2015Esc, Chi2017How, Hadi2018Esc, Harsh2021Esc}, achieve better generalization \citep{Moritz2016Tra, Wen2018Gen}, and find a local minimum with a small loss value in polynomial time under some assumptions \citep{Yuchen2017AHi}. Several studies have also shown that performance can be improved by adding artificial noise to gradient descent (GD) \citep{Rong2015Esc, Zhou2019Tow, Jin2021OnN, Orvieto2022Ant}. \citet{Robert2018AnA} also suggests that noise smoothes the objective function. Here, at time $t$, let $\bm{y}_{t}$ be the parameter updated by GD and $\bm{x}_{t+1}$ be the parameter updated by SGD, i.e., 
\begin{align*}
\bm{y}_{t} := \bm{x}_t - \eta \nabla f(\bm{x}_t),\ 
\bm{x}_{t+1} := \bm{x}_t - \eta \nabla f_{\mathcal{S}_t}(\bm{x}_t).
\end{align*}
Then, the update rule for the sequence $\left\{ \bm{y}_t \right\}$ is as follows:
\begin{align} \label{eq:1}
\mathbb{E}_{\bm{\omega}_t} \left[\bm{y}_{t+1} \right] 
= \mathbb{E}_{\bm{\omega}_t} \left[ \bm{y}_t \right] - \eta \nabla \mathbb{E}_{\bm{\omega}_t} \left[f(\bm{y}_t - \eta \bm{\omega}_t) \right],
\end{align}
where $f$ is Lipschitz continuous and differentiable (The derivation of equation (\ref{eq:1}) is in Appendix \ref{sec:der}). Therefore, if we define a new function $\hat{f}(\bm{y}_t) := \mathbb{E}_{\bm{\omega}_t}[f(\bm{y}_t - \eta \bm{\omega}_t)]$, $\hat{f}$ can be smoothed by convolving $f$ with noise (see Definition \ref{dfn:fhat}, also \citep{Zhijun1996The}), and its parameters $\bm{y}_t$ can approximately be viewed as being updated by using the gradient descent to minimize $\hat{f}$. In other words, simply using SGD with a mini-batch smoothes the function to some extent and may enable escapes from local minima.

\paragraph{Graduated Optimization}\label{sec:1.1.1}
Graduated optimization is one of the global optimization methods that search for the global optimal solution of difficult multimodal optimization problems. The method generates a sequence of simplified optimization problems that gradually approach the original problem through different levels of local smoothing operations. It solves the easiest simplified problem first, as the easiest simplification should have nice properties such as convexity or strong convexity; after that, it uses that solution as the initial point for solving the second-simplest problem, then the second solution as the initial point for solving the third-simplest problem and so on, as it attempts to escape from local optimal solutions of the original problem and reach a global optimal solution.

This idea first appeared in the form of graduated non-convexity (GNC) by \citep{Andrew1987Vis} and has since been studied in the field of computer vision for many years. Similar early approaches can be found in \citep{Andrew1987Sig} and \citep{Yuille1989Ene}. Moreover, the same concept has appeared in the fields of numerical analysis \citep{Eugene1990Num} and optimization \citep{Kenneth1990Ade, Zhijun1996The}. Over the past 30 years, graduated optimization has been successfully applied to many tasks in computer vision, such as early vision \citep{Michael1996Ont}, image denoising \citep{Nikolova2010Fas}, optical flow \citep{Sun2010Sec, Thomas2011Lar}, dense correspondence of images \citep{Kim2013Def}, and robust estimation \citep{Heng2020Gra, Antonante2022Out, Peng2023Ont}. In addition, it has been applied to certain tasks in machine learning, such as semi-supervised learning \citep{Olivier2006Aco, Vikas2006Det, Olivier2008Opt}, unsupervised learning \citep{Noah2004Ann}, and ranking \citep{Olivier2010Gra}. Moreover, score-based generative models \citep{Song2019Gen, Song2021Sco} and diffusion models \citep{Dickstein2015Dee, Ho2020Den, Song2021Den, Robin2022Hig}, which are currently state-of-the-art generative models, implicitly use the techniques of graduated optimization. A comprehensive survey on the graduated optimization approach can be found in \citep{Hossein2015Ont}.

Several previous studies have theoretically analyzed the graduated optimization algorithm. \citet{Hossein2015ATh} performed the first theoretical analysis, but they did not provide a practical algorithm. \citet{Elad2016OnG} defined a family of nonconvex functions satisfying certain conditions, called $\sigma$-nice, and proposed a first-order algorithm based on graduated optimization. In addition, they studied the convergence and convergence rate of their algorithm to a global optimal solution for $\sigma$-nice functions. \citet{Iwakiri2022Sin} proposed a single-loop method that simultaneously updates the variable that defines the noise level and the parameters of the problem and analyzed its convergence. 
\citet{Li2023Sto} analyzed graduated optimization based on a special smoothing operation. Note that \citet{Duchi2012Ran} pioneered the theoretical analysis of optimizers using Gaussian smoothing operations for nonsmooth convex optimization problems. Their method of optimizing with decreasing noise level is truly a graduated optimization approach.
Recently, \citet{Sato2025exp} considered implicit graduated optimization using smoothing of the objective function by stochastic noise in SGD. 
Here, this paper refers to the graduated optimization approach with explicit smoothing operations (Definition \ref{dfn:fhat}) as ``explicit graduated optimization'' and to the graduated optimization approach with implicit smoothing operations as ``implicit graduated optimization''. 

\subsection{Motivation}
Equation (\ref{eq:1}) indicates that SGD smoothes the objective function \citep{Robert2018AnA}, but it is not clear to what extent the function is smoothed or what factors are involved in the smoothing. Therefore, we decided to clarify these aspects and identify what parameters contribute to the smoothing. Also, once it is known what parameters of SGD contribute to smoothing, an implicit graduated optimization can be achieved by varying the parameters so that the noise level is reduced gradually. Our goal was thus to construct an implicit graduated optimization framework using the smoothing properties of SGD to achieve global optimization of deep neural networks.

\subsection{Contributions}
\UPDATE{
\paragraph{Generalized Function Smoothing (Section \ref{sec:func})} Traditionally, Gaussian noise has been used for function smoothing, but the conditions the noise must satisfy for smoothing have never been discussed. We show that when the function $f$ satisfies $f(\bm{x}) = \mathcal{O}\left( \|\bm{x}\|^p \right)$ for some $p \geq 0$, light-tailed noise such as Gaussian noise can always smooth the function. Furthermore, we find that even heavy-tailed noise can smooth the function if its tail index $\alpha>0$ satisfies $\alpha > p$ (Proposition \ref{prop:heavy}). This finding is a valuable result that leads to an understanding of function smoothing by stochastic noise when the stochastic noise in SGD follows a heavy-tailed distribution, as discussed in Section \ref{sec:3.2}

\begin{figure*}[htbp]
\begin{minipage}[t]{0.99\textwidth}
\vspace*{-10pt}
\centering
\includegraphics[width=0.8\linewidth]{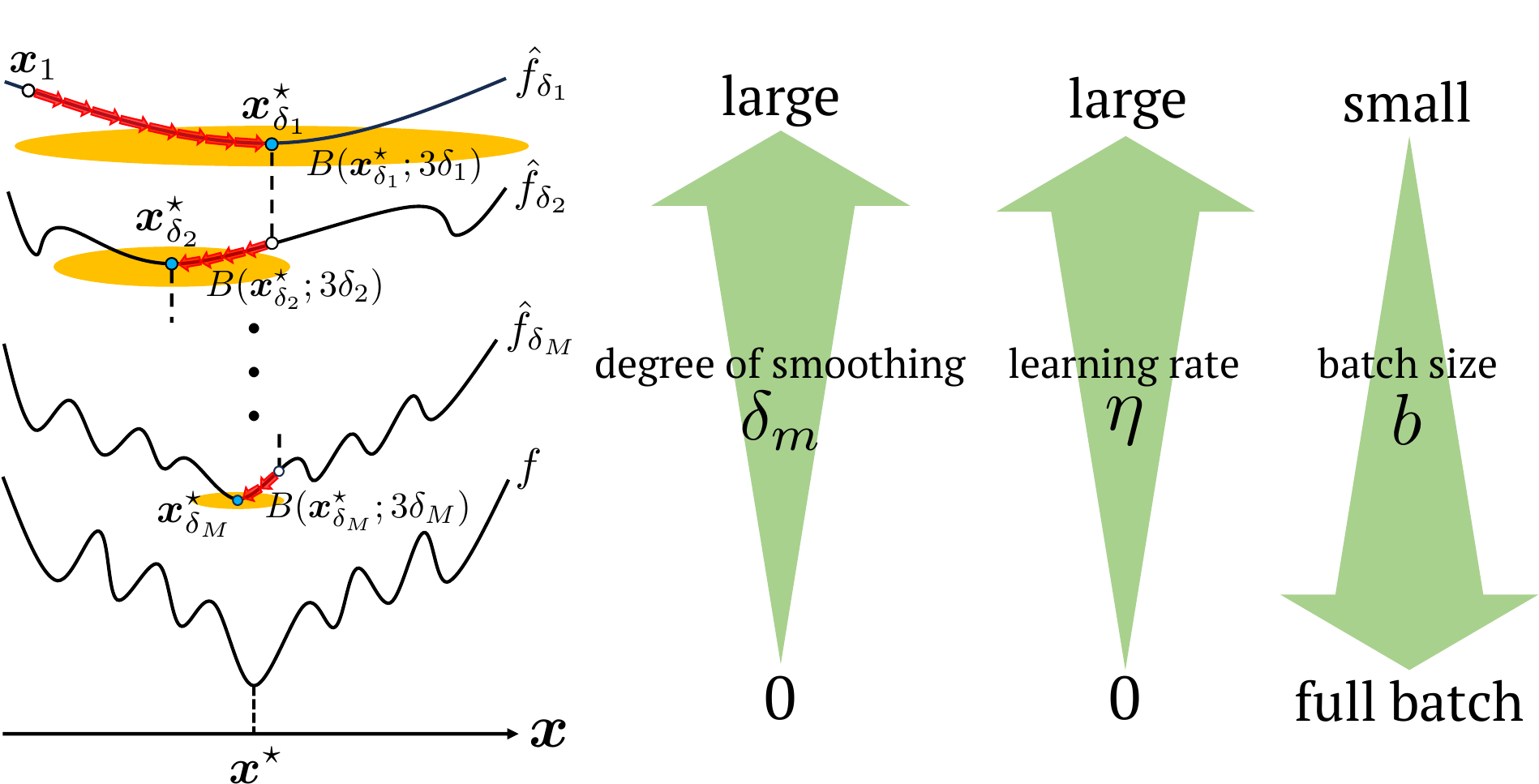}%
\end{minipage}%
\vspace*{-5pt}
\caption{Conceptual diagram of implicit graduated optimization for weakly $(\mu_1, \mu_2)$-nice function.}%
\vspace*{-10pt}
\label{fig:00}%
\end{figure*}%

\paragraph{SGD's Smoothing Property under Heavy-tailed Noise (Section \ref{sec:3.2})} We show that the degree of smoothing provided by SGD's stochastic noise depends on the quantity $\delta = \frac{\eta C_\alpha^{1/\alpha} \tau}{b^{(\alpha-1)/\alpha}}$, where $\eta$ is the learning rate, $b$ is the batch size, $\tau^\alpha$ for some $\alpha \in (1,2]$ is the upper bound of the $\alpha$-th moment of the stochastic noise (see Assumption \ref{assum:03}), and $C_\alpha \in [1,2)$ is a constant that depends only on $\alpha$ (see Lemma \ref{lem:vBE}).}
Accordingly, the smaller the batch size $b$ is and the larger the learning rate $\eta$ is, the smoother the function becomes (see Figure \ref{fig:00}). This finding provides a theoretical explanation for several experimental observations. For example, as is well known, training with a large batch size leads to poor generalization performance, as evidenced by the fact that several prior studies \citep{Hoffer2017Tra, Goyal2017Acc, You2020Lar} provided techniques that do not impair generalization performance even with large batch sizes. This is because, if we use a large batch size, the degree of smoothing $\delta$ becomes smaller and the original nonconvex function is not smoothed enough, so the sharp local minima do not disappear and the optimizer is more likely to fall into one. \citet{Shirish2017OnL} showed this experimentally, and our results provide theoretical support for it.


\UPDATE{
\paragraph{Relaxation of $\sigma$-nice Function and Implicit Graduated Optimization (Section \ref{sec:5})} Since the degree of smoothing of the objective function by stochastic noise in SGD is determined by $\delta = \frac{\eta C_\alpha^{1/\alpha} \tau}{b^{(\alpha-1)/\alpha}}$, it should be possible to construct an implicit graduated optimization algorithm by decreasing the learning rate $\eta$ and/or increasing the batch size $b$ during training (see Figure \ref{fig:00}). On the basis of this theoretical intuition, we propose a new implicit graduated optimization algorithm. Moreover, in analyzing this algorithm, we propose a weakly $(\mu_1, \mu_2)$-nice function (Definition \ref{dfn:weak}), which significantly relaxes the standard analysis tool of the $\sigma$-nice function (Definition \ref{dfn:sigma}). In contrast to the $\sigma$-nice function, this does not require local (strong) convexity for the smoothed function, allowing consideration of a broader class of nonconvex functions than the $\sigma$-nice function.}
We show that the algorithm for the weakly $(\mu_1, \mu_2)$-nice function converges to an $\epsilon$-neighborhood of the global optimal solution in $\Omega \left(1/ \epsilon^{2}\right)$ rounds (Theorem \ref{thm:3.4}).

\section{Preliminaries}
\label{sec:pre}
\paragraph{Notation and Assumptions}
Let $\mathbb{N}$ be the set of non-negative integers. For $m \in \mathbb{N} \setminus \{0\}$, define $[m] := \{1,2,\ldots,m\}$. Let $\mathbb{R}^d$ be a $d$-dimensional Euclidean space with inner product $\langle \cdot, \cdot \rangle$, which induces the norm $\| \cdot \|$. $I_d$ denotes a $d \times d$ identity matrix. $B(\bm{y}; r)$ is the Euclidean closed ball of radius $r$ centered at $\bm{y}$, i.e., $B(\bm{y};r) := \left\{ \bm{x} \in \mathbb{R}^d : \| \bm{x} - \bm{y} \| \leq r \right\}$. Let $\mathcal{N}(\bm{\mu}; \Sigma)$ be a $d$-dimensional Gaussian distribution with mean $\bm{\mu} \in \mathbb{R}^d$ and variance $\Sigma \in \mathbb{R}^{d \times d}$. The DNN is parameterized by a vector $\bm{x} \in \mathbb{R}^d$, which is optimized by minimizing the empirical loss function $f(\bm{x}) := \frac{1}{n} \sum_{i \in [n]} f_i(\bm{x})$, where $f_i(\bm{x})$ is the loss function for $\bm{x} \in \mathbb{R}^d$ and the $i$-th training data $\bm{z}_i$ $(i \in [n])$. Let $\xi$ be a random variable that does not depend on $\bm{x} \in \mathbb{R}^d$, and let $\mathbb{E}_{\xi}[X]$ denote the expectation with respect to $\xi$ of a random variable $X$. $\xi_{t,i}$ is a random variable generated from the $i$-th sampling at time $t$, and $\bm{\xi}_t := (\xi_{t,1}, \xi_{t,2}, \ldots, \xi_{t,b})^\top$ is independent of $(\bm{x}_k)_{k=0}^{t} \subset \mathbb{R}^d$ generated by SGD, where $b$ $(\leq n)$ is the batch size. The independence of $\bm{\xi}_0, \bm{\xi}_1, \ldots$ allows us to define the total expectation $\mathbb{E}$ as $\mathbb{E}=\mathbb{E}_{\bm{\xi}_0} \mathbb{E}_{\bm{\xi}_1} \cdots \mathbb{E}_{\bm{\xi}_t}$. Let $\mathsf{G}_{\xi}(\bm{x})$ be the stochastic gradient of $f(\cdot)$ at $\bm{x} \in \mathbb{R}^d$. The mini-batch $\mathcal{S}_t$ consists of $b$ samples at time $t$, and the mini-batch stochastic gradient of $f(\bm{x}_t)$ for $\mathcal{S}_t$ is defined as $\nabla f_{\mathcal{S}_t}(\bm{x}_t) := \frac{1}{b} \sum_{i \in [b]} \mathsf{G}_{\xi_{t,i}}(\bm{x}_t) = \frac{1}{b}\sum_{i \in \mathcal{S}_t} \nabla f_i (\bm{x}_t)$.

\begin{assum}\label{assum:01}
$f \colon \mathbb{R}^d \to \mathbb{R}$ is continuously differentiable and $L_g$-smooth, i.e., for all $\bm{x}, \bm{y} \in \mathbb{R}^d$, 
\begin{align*}
\| \nabla f(\bm{x}) - \nabla f(\bm{y}) \| \leq L_g \| \bm{x} - \bm{y} \|.
\end{align*}
\end{assum}

\begin{assum}\label{assum:02}
$f \colon \mathbb{R}^d \to \mathbb{R}$ is an $L_f$-Lipschitz function, i.e., for all $\bm{x}, \bm{y} \in \mathbb{R}^d$, 
\begin{align*}
\left|f(\bm{x}) - f(\bm{y}) \right| \leq L_f \| \bm{x} - \bm{y} \|.
\end{align*}
\end{assum}

\begin{assum}\label{assum:03}
{\em (i)} For all iteration $t \in \mathbb{N}$ and all index $i \in [n]$, 
\begin{align*}
\mathbb{E}_{\xi_{t,i}}\left[ \mathsf{G}_{\xi_{t,i}}(\bm{x}_t) \right] = \nabla f(\bm{x}_t).
\end{align*}

{\em (ii)} There exist $\tau \geq 0$ and $\alpha \in (1,2]$ such that, for all $t \in \mathbb{N}$ and all $i \in [n]$,
\begin{align*}
\mathbb{E}_{\xi_{t,i}}\left[ \| \mathsf{G}_{\xi_{t,i}}(\bm{x}_t) - \nabla f(\bm{x}_t) \|^\alpha \right] \leq \tau^\alpha.
\end{align*}
\end{assum}

\UPDATE{Assumption \ref{assum:03}(ii) symbolizes the analysis of optimizers under heavy-tailed noise and is widely used in previous studies \citep{Zhang2020Why, Cutkosky2021Hig, Sadiev2023Hig, Nguyen2023Imp, Chezhegov2025Cli}. When $\alpha=2$, it reduces to the standard bounded variance assumption \citep{Nemirovski2009Rob, Ghadimi2012Opt, Ghadimi2013Sto}; when $\alpha < 2$, the stochastic gradient may possess unbounded variance.}

\section{Theory of function smoothing}\label{sec:func}
\UPDATE{Classically, function smoothing is achieved by convolving the function with a random variable following a Gaussian distribution.}
\begin{dfn}
[Classical Smoothed function] \label{dfn:fhat} Given a function $f \colon \mathbb{R}^d \to \mathbb{R}$, define $\hat{f}_\delta \colon \mathbb{R}^d \to \mathbb{R}$ for all $\bm{x} \in \mathbb{R}^d$ to be the function obtained by smoothing $f$ as 
\begin{align*}
\hat{f}_\delta(\bm{x}) := \mathbb{E}_{\bm{u} \sim \mathcal{N}\left(\bm{0}, \frac{1}{d}I_d\right)} \left[f(\bm{x} - \delta \bm{u}) \right], 
\end{align*}
where $\delta > 0$ represents the degree of smoothing and $\bm{u}$ is a random variable from a Gaussian distribution $\mathcal{N}\left(\bm{0}, \frac{1}{d}I_d\right)$ with $\mathbb{E}_{\bm{u}} \left[ \| \bm{u} \| \right] \leq 1$.
\end{dfn}
\UPDATE{Function smoothing was introduced into the optimization context by \citep{Zhijun1996The} by applying Gaussian kernel blurring to images in computer vision. As its history indicates, it is highly heuristic and remains poorly understood. For example, the distribution followed by the random variables used for smoothing varies across the literature, sometimes being uniform \citep{Elad2016OnG} and other times Gaussian \citep{Zhijun1996The, Hossein2015Ont, Iwakiri2022Sin}. So what probability distribution should the random variable $\bm{u}$ follow in order to smooth the function? This has never been discussed. In this section, we examine what conditions are necessary for smoothing functions, by focusing not only on light-tailed distributions such as Gaussian and uniform distributions but also on heavy-tailed distributions. In particular, we investigate whether functions can be smoothed even for heavy-tailed distributions with a power-law decay \citep{Mert2021The}, an important subclass of heavy-tailed distributions. They are defined as $P(X>t) \sim c_0 t^{-\alpha}$ as $t \to \infty$ for some $c_0>0$ and $\alpha>0$, where $\alpha$ is known as the tail index, which determines the tail thickness of the distribution. This class includes the Pareto distribution and the $\alpha$-stable distribution ($\alpha \in (0,2)$).

We present motivated experimental results that aid in the intuitive understanding of our subsequent theory. We introduce one-dimensional versions of Rastrigin's function \citep{Torn1989Glo, Rudolph1990Glo} and the drop-wave function \citep{Molga2005Tes}, which are well-known test functions for global optimization algorithms.
\vspace*{-3pt}
\begin{align}
&\text{(Rastrigin's function)}\ \ f(x) := x^2 - 10\cos(2\pi x)+10, \label{eq:R}\\
&\text{(Drop-Wave function)}\ \ f(x) := -\frac{1 + \cos(12\pi x)}{0.5 x^2 + 2}. 
\label{eq:DW}
\end{align}
We smooth the above functions according to Definition \ref{dfn:fhat} by using random variables following light-tailed distributions: Gaussian, uniform, exponential, and Rayleigh, and heavy-tailed distributions: Pareto, Cauchy, and Levy. Figure \ref{fig:smoothed} compares Rastrigin's and drop-wave functions with their smoothed versions using a degree of smoothing $\delta = 0.5$, where the stochastic noise is drawn from either a light-tailed ((a) and (c)) or a heavy-tailed distribution ((b) and (d)). For Rastrigin's function, smoothing with a light-tailed distribution is effective, whereas smoothing with a heavy-tailed distribution fails, likely because of the presence of extremely large samples that lead to unstable function values. In contrast, for the drop-wave function, both light-tailed and heavy-tailed distributions successfully smooth the function. This robustness stems from the $x^2$ term in the denominator, which suppresses large function values and prevents instability even under heavy-tailed noise.
}
\begin{figure*}[htbp]
\vspace*{-3pt}
\begin{minipage}[t]{0.99\textwidth}
\centering
\includegraphics[width=1.0\linewidth]{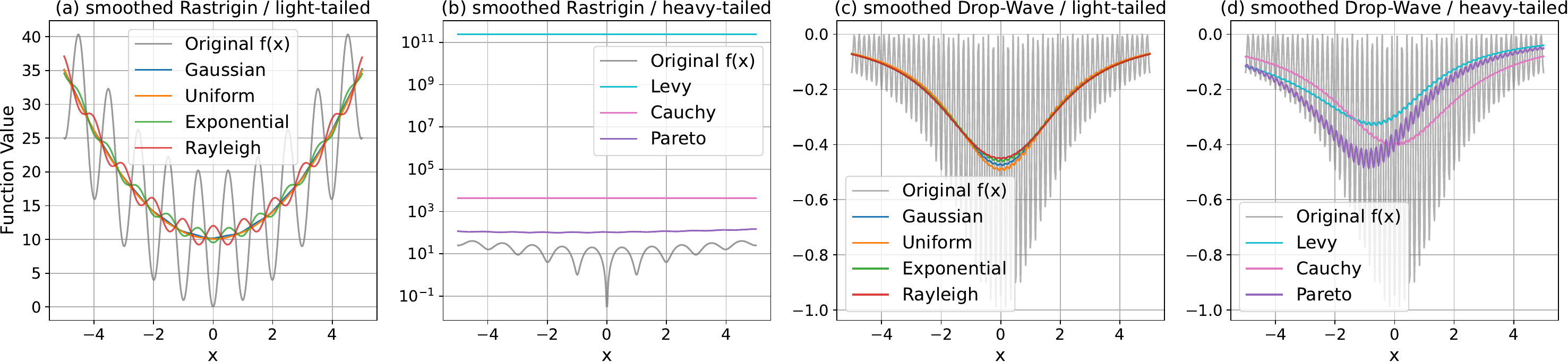}%
\end{minipage}
\vspace*{-5pt}
\caption{\UPDATE{Rastrigin and drop-wave functions with their smoothed versions under light-tailed ((a), (c)) and heavy-tailed ((b), (d)) random perturbations.}}
\vspace*{-5pt}
\label{fig:smoothed}
\end{figure*}

\UPDATE{These experimental results suggest that whether a function can be smoothed using heavy-tailed noise depends on the nature of the function itself. We elevate this observation to the following proposition. The proof of Proposition \ref{prop:heavy} is in Appendix \ref{subsec:e.1}.

\begin{prop}
\label{prop:heavy}
Assume that the function $f \colon \mathbb{R}^d \to \mathbb{R}$has the property that there exists $p \geq 0$ such that, for all $\bm{x} \in \mathbb{R}^d$, $|f(\bm{x})| = \mathcal{O}\left( \| \bm{x} \|^p \right)$. Let $\bm{x} \in \mathbb{R}^d$ and $\delta>0$. Then,
\begin{align*}
\mathbb{E}_{\bm{u}}\left[ |f(\bm{x}-\delta \bm{u})| \right] < \infty
\text{ holds for}
\begin{cases}
p \geq 0 &\text{if $\bm{u}$ follows a light-tailed distribution,} \\
p < \alpha &\text{if $\bm{u}$ follows a heavy-tailed distribution with tail index $\alpha$.}
\end{cases}
\end{align*}
\end{prop}
This proposition provides a sufficient condition for defining a smoothed function. That is, Proposition \ref{prop:heavy} implies that even if $\bm{u}$ follows a heavy-tailed distribution, the smoothed version $\hat{f}_{\delta}$ can be defined if $\alpha > p$. Therefore, we can extend the definition of function smoothing as follows.

\begin{dfn}
[Generalized Smoothed function] \label{dfn:newfhat} For a function $f \colon \mathbb{R}^d \to \mathbb{R}$ satisfying $|f(\bm{x})| = \mathcal{O}\left( \| \bm{x} \|^p \right)$ for some $p\geq0$, let $\mathcal{L}$ be either a light-tailed distribution or a heavy-tailed distribution with tail index $\alpha > p$. Then, we define $\hat{f}_\delta \colon \mathbb{R}^d \to \mathbb{R}$ to be the function obtained by smoothing $f$ as follows: 
\begin{align*}
\hat{f}_\delta(\bm{x}) := \mathbb{E}_{\bm{u} \sim \mathcal{L}} \left[f(\bm{x} - \delta \bm{u}) \right], 
\end{align*}
where $\delta > 0$ represents the degree of smoothing and $\bm{u}$ is a random variable with $\mathbb{E}_{\bm{u}} \left[ \| \bm{u} \| \right] \leq 1$.
\label{dfn:newfhat}
\end{dfn}
Smoothed functions obeying Definition \ref{dfn:newfhat} have the following well-known properties, which hold for smoothed functions obeying Definition \ref{dfn:fhat}. Crucially, these properties hold regardless of the distribution of $\bm{u}$. The proof of Lemma $\ref{lem:04}$ is in Appendix \ref{sec:C}.}

\begin{lem}\label{lem:04}
Suppose that Assumptions \ref{assum:01} and \ref{assum:02} holds; then, $\hat{f}_\delta$ defined by Definition \ref{dfn:newfhat} is also $L_g$-smooth and $L_f$-Lipschitz; i.e., for all $\bm{x}, \bm{y} \in \mathbb{R}^d$,
\begin{align*}
\left\| \nabla \hat{f}_\delta(\bm{x}) - \nabla \hat{f}_\delta(\bm{y}) \right\| 
\leq L_g \| \bm{x} - \bm{y} \|, \ \text{ and }\ 
\left|\hat{f}_\delta(\bm{x}) - \hat{f}_\delta(\bm{y}) \right| 
\leq L_f \| \bm{x} - \bm{y} \|.
\end{align*}
\end{lem}


Lemma \ref{lem:04} imply that the Lipschitz constants $L_f$ of the original function $f$ and $L_g$ of $\nabla f$ are taken over by the smoothed function $\hat{f}_\delta$ and its gradient $\nabla \hat{f}_\delta$ for all $\delta \in \mathbb{R}$. \UPDATE{Proposition \ref{prop:heavy} merely indicates the minimal conditions for defining $\hat{f}_{\delta}$. However, these lemmas guarantee that, once defined, $\hat{f}_{\delta}$ possesses at least as much smoothness as $f$, thereby establishing smoothing with heavy-tailed noise. Note that the last paragraph of Section \ref{sec:3.2} discuses whether the commonly used empirical loss function in machine learning can be smoothed even with heavy-tailed noise.}

\section{SGD's smoothing property}\label{sec:3.2}
\UPDATE{This section discusses the smoothing effect of using stochastic gradients. We assume that the original function $f$ can be smoothed even by heavy-tailed noise with tail index $\alpha$, i.e., $f(\bm{x})=\mathcal{O}\left( \| \bm{x} \|^p \right)$ for some $p\geq0$ such that $\alpha > p$. First, we introduce the lemma needed to evaluate the magnitude of the stochastic noise in SGD.
The following lemma is very important to our theory. It reduces to the well-known result, $\mathbb{E}_{\bm{\xi}_t}\left[ \| \nabla f_{\mathcal{S}_t}(\bm{x}_t) - \nabla f(\bm{x}_t) \|^2 \right] \leq \frac{\tau^2}{b}$, when $\alpha=2$, and thus constitutes an extension of that result. The proof of Lemma \ref{lem:00} is in Appendix \ref{sec:C.3}.
\begin{lem}\label{lem:00}
Suppose that Assumption \ref{assum:03} holds. Then, for all $t \in \mathbb{N}$, 
\begin{align*}
\mathbb{E}_{\bm{\xi}_t} \left[ \| \nabla f_{\mathcal{S}_t}(\bm{x}_t) - \nabla f(\bm{x}_t) \|^\alpha \right] 
\leq \frac{C_\alpha \tau^\alpha}{b^{\alpha-1}}, 
\end{align*}
where $C_\alpha$ is continuously and strictly decreasing in $\alpha \in (1,2]$ from $\displaystyle \lim_{\alpha \to1} C_\alpha = 2$ to $C_2 = 1$.
\end{lem}}

From Lemma \ref{lem:00}, we have
$
\mathbb{E}_{\xi_t} \left[ \left\| \bm{\omega}_t \right\| \right]
\leq  \frac{C_\alpha^{1/\alpha} \tau}{b^{(\alpha-1)/\alpha}},
$
due to $\bm{\omega}_t := \nabla f_{\mathcal{S}_t} (\bm{x}_t) - \nabla f(\bm{x}_t)$. The $\bm{\omega}_t$ for which this equation is satisfied can be expressed as $\bm{\omega}_t =  \frac{C_\alpha^{1/\alpha} \tau}{b^{(\alpha-1)/\alpha}}\bm{u}_t,$ where $\mathbb{E}_{\xi_t}\left[ \| \bm{u}_t \| \right] \leq 1$. In addition, for analysis, we also assume the stochastic noise $\bm{\omega}_t$ follows the same probability distribution at any given time $t$. Note that the experimental results of \citep[Figure 2(b)]{Zhang2020Why} partially justify this assumption. Therefore, $\bm{\omega}_t \sim \hat{\mathcal{L}}$ and thereby $\bm{u}_t \sim \mathcal{L}$, where $\hat{\mathcal{L}}$ and $\mathcal{L}$ are light-tailed distributions or favorable heavy-tailed distributions and $\mathcal{L}$ is a scaled version of $\hat{\mathcal{L}}$. Then, using Definition \ref{dfn:fhat}, we further transform equation (\ref{eq:1}) as follows:
\begin{align*}
\mathbb{E}_{\bm{\omega}_t} \left[\bm{y}_{t+1} \right] 
&= \mathbb{E}_{\bm{\omega}_t} \left[ \bm{y}_t \right] - \eta \nabla \mathbb{E}_{\bm{\omega}_t} \left[f(\bm{y}_t - \eta \bm{\omega}_t) \right] \nonumber \\
&= \mathbb{E}_{\bm{\omega}_t} \left[ \bm{y}_t \right] - \eta \nabla \mathbb{E}_{\bm{u}_t \sim \mathcal{L}} \left[f\left(\bm{y}_t - \frac{\eta C_\alpha^{1/\alpha} \tau}{b^{(\alpha-1)/\alpha}}\bm{u}_t\right) \right] \nonumber \\
&= \mathbb{E}_{\bm{\omega}_t} \left[ \bm{y}_t \right]- \eta \nabla \hat{f}_{\delta}(\bm{y}_t),%
\end{align*}
where $\delta := \frac{\eta C_\alpha^{1/\alpha} \tau}{b^{(\alpha-1)/\alpha}}$. In addition, from $\mathbb{E}_{\bm{\omega}_t}[\bm{x}_{t+1}] = \bm{y}_t$ (see Appendix \ref{sec:der}), we have
\begin{align*}
\mathbb{E}_{\bm{\omega}_t}[\bm{y}_t]
= \mathbb{E}_{\bm{\omega}_t}[\mathbb{E}_{\bm{\omega}_{t}}[\bm{x}_{t+1}]] 
= \mathbb{E}_{\bm{\omega}_{t}}[\bm{x}_{t+1}]
= \bm{y}_t.
\end{align*}
Therefore, we have
\begin{align}
\bm{y}_{t+1} = \boldsymbol{y}_t - \eta \nabla \hat{f}_{\delta}(\boldsymbol{y}_t) \UPDATE{+ \bm{\nu}_t}, \text{ where } \delta := \frac{\eta C_\alpha^{\frac{1}{\alpha}} \tau}{b^{\frac{\alpha-1}{\alpha}}} \text{ and } \UPDATE{\bm{\nu}_t := \bm{y}_{t+1} - \mathbb{E}_{\bm{\omega}_t}[\bm{y}_{t+1}]}. \label{eq:27}
\end{align}
\UPDATE{Note that $\mathbb{E}_{\bm{\omega}_t}[\bm{\nu}_{t}] = \bm{0}$ and $\mathbb{E}_{\bm{\omega}_t}\left[ \| \bm{\nu}_t \|^\alpha \right] \leq 2^{\alpha-1} (1+2^\alpha L_g^\alpha) \delta^\alpha$ holds (see Lemma \ref{lem:nu}) and }recall that $\bm{x}_{t+1} := \bm{x}_t - \eta \nabla f_{\mathcal{S}_t}(\bm{x}_t)$ and $\bm{y}_t := \bm{x}_t - \eta \nabla f(\bm{x}_t)$. Even though $\bm{y}_t$ is defined using the gradient of $f$, equation (\ref{eq:27}) implies that $\bm{y}_t$ can be updated by using the GD \UPDATE{with noise} to minimize $\hat{f}_{\delta}$ and optimizing the function $f$ with SGD is equivalent to optimizing the function $\hat{f}_{\delta}$ with GD \UPDATE{with noise}. Therefore, we can say that the degree of smoothing $\delta$ by the stochastic noise $\bm{\omega}_t$ in SGD is determined by the learning rate $\eta$, the batch size $b$, the $\alpha$-th moment of the stochastic noise $\tau^\alpha$, and tail index $\alpha$. Note that, when $\alpha=2$, the degree of smoothing reduces to $\delta = \eta \tau/\sqrt{b}$, making this an extension of existing results \citep{Sato2025exp}.


There are still more discoveries that can be made from the finding that simply by using SGD for optimization, the objective function is smoothed and the degree of smoothing is determined by $\delta$.

\noindent
\paragraph{1. Why the Use of Large Batch Sizes Leads to Solutions Falling into Sharp Local Minima} It is known that training with large batch sizes leads to a persistent degradation of model generalization performance. In particular, \citet{Shirish2017OnL} showed experimentally that learning with large batch sizes leads to sharp local minima and worsens generalization performance. According to equation (\ref{eq:27}), using a large learning rate and/or a small batch size will make the function smoother. Thus, in using a small batch size, the sharp local minima will disappear through extensive smoothing, and SGD can reach a flat local minimum. Conversely, when using a large batch size, the smoothing is weak and the function is close to the original multimodal function, so it is easy for the solution to fall into a sharp local minimum. Thus, we have theoretical support for what \citet{Shirish2017OnL} showed experimentally, and our experiments have yielded similar results (see Figure \ref{fig:999} (a)). 

\noindent
\paragraph{2. Why Decaying Learning Rates and Increasing Batch Sizes are Superior to Fixed Learning Rates and Batch Sizes} From equation (\ref{eq:27}), the use of a decaying learning rate or increasing batch size during training is equivalent to decreasing the noise level of the smoothed function, so using a decaying learning rate or increasing the batch size is an implicit graduated optimization. Thus, we can say that using a decaying learning rate \citep{Ilya2017SGD, Hundt2019sha, You2019How, Aitor2021How} or increasing batch size \citep{Richard2012Sam, Michael2012Hyb, Balles2017Cou, De2017Aut, Bottou2018Opt, Samuel2018Don} makes sense in terms of avoiding local minima and provides theoretical support for their experimental superiority.

\noindent
\paragraph{\UPDATE{Discussion on Smoothing Empirical Loss Function}} \UPDATE{In this section, we assumed that the growth order $p > 0$ of the original function $f$ satisfies $\alpha > p$ for the tail index $\alpha$ of heavy-tailed noise. Does the empirical loss function defined by the output from deep neural networks and the true labels satisfy this condition? From several prior studies on homogeneity \citep{Neyshabur2015Pat, Du2018Alg, Lyu2020Gra}, we can show that in the case of deep fully-connected networks or CNNs with ReLU or LeakyReLU activations, with all bias terms removed, the growth order $p$ of the cross-entropy loss is equal to the number of layers (see Appendix \ref{sec:homo}). Therefore, when using sufficiently deep models, this condition $\alpha > p$ can be said to be theoretically violated. On the other hand, according to \citep{Wan2021Sph}, when commonly used normalization techniques like batch normalization or group normalization are incorporated into the model, the parameter exhibits scale-invariance with respect to the empirical loss $f$. This immediately implies that the growth order of $f$ is zero (see Appendix \ref{sec:homo}), thus satisfying $\alpha > p = 0$. Since normalization techniques are standard practice for current DNN models, it can be said that the empirical loss function can be smoothed even by heavy-tailed noise. Consequently, this insight dramatically expands the applicability of the theory of smoothing via stochastic noise in SGD and significantly enhances its value.}

\section{Implicit Graduated Optimization}
\label{sec:5}
\UPDATE{Let us briefly explain the notation.}
The graduated optimization algorithm uses several smoothed functions with different noise levels. There are a total of $M$ noise levels $(\delta_m)_{m \in [M]}$ and smoothed functions $(\hat{f}_{\delta_m})_{m \in [M]}$ in this paper. The largest noise level is $\delta_1$ and the smallest is $\delta_{M}$ (see also Figure \ref{fig:00}). For all $m \in [M]$, $(\hat{\bm{x}}_t^{(m)})_{t \in \mathbb{N}}$ is the sequence generated by an optimizer to minimize $\hat{f}_{\delta_m}$. Also, we denote 
$
\bm{x}^\star := \underset{\bm{x} \in \mathbb{R}^d} {\operatorname{argmin}} \ f(\bm{x}) \text{\ \ and \ } \bm{x}_{\delta}^\star := \underset{\bm{x} \in \mathbb{R}^d} {\operatorname{argmin}} \ \hat{f}_\delta (\bm{x}).
$

\subsection{Relaxation of $\sigma$-nice Function}
In order to analyze the graduated optimization algorithm, Hazan et al. defined $\sigma$-nice functions, a family of nonconvex functions that has favorable conditions for a graduated optimization algorithm to converge to a global optimal solution \citep{Elad2016OnG}. 

\begin{dfn}
[$\sigma$-nice function \citep{Elad2016OnG}]\label{dfn:sigma} Let $M \in \mathbb{N}$, $m \in [M]$, and $\delta_{m+1} := \frac{\delta_m}{2}$. A function $f \colon \mathbb{R}^d \to \mathbb{R}$ is said to be $\sigma$-nice if the following two conditions hold:

{\em (i)} For all $\delta_m > 0$ and all $\bm{x}_{\delta_m}^\star$, there exists $\bm{x}_{\delta_{m+1}}^\star$ such that: $\left\| \bm{x}_{\delta_m}^\star - \bm{x}_{\delta_{m+1}}^\star \right\| \leq \frac{\delta_m}{2}$.

{\em (ii)} For all $\delta_m>0$, the function $\hat{f}_{\delta_m}(\bm{x})$ over $B(\bm{x}_{\delta_m}^\star;3\delta_m)$ is $\sigma$-strongly convex.
\end{dfn}
Careful examination of their analysis reveals that the necessary conditions for the success of graduated optimization are: (1) $\hat{f}_{\delta}$ has a unique stationary point that is a global optimal solution $\bm{x}_{\delta}^\star$, (2) the optimizer (GD in our case) has a guarantee of convergence to the stationary point of $\hat{f}_{\delta}$, and (3) the optimal solution obtained at the $m$-th stage $\bm{x}_{\delta_m}^\star$ serves as a good starting point for the $(m+1)$-th stage. Condition (3) is guaranteed by the $\sigma$-nice property (i). To satisfy conditions (1) and (2), Hazan et al. adopt local strong convexity as a $\sigma$-nice property (ii), but we show that this property can be replaced with a more relaxed one as follows: there exist $\mu_1, \mu_2 > 0$ such that, for all $\bm{x} \in \mathbb{R}^d$,
\vspace*{-10pt}
\begin{align*} 
\text{($\mu_1$-PL condition)}&\quad \frac{1}{2}\| \nabla g(\bm{x}) \|^2 \geq \mu_1(g(\bm{x}) - g^\star), \\
\text{(unique $\mu_2$-QG condition)}&\quad g(\bm{x}) - g^\star \geq \frac{\mu_2}{2}\| \bm{x} - \bm{x}^\star \|^2,
\vspace*{-6pt}
\end{align*}
where $\bm{x}^\star$ is a global minimizer of $g$ and $g^\star := g(\bm{x}^\star)$. The Polyak-\L{}ojasiewicz (PL) condition \citep{Polyak1963Gra} guarantees that all stationary points of the function $g$ are global optimal solutions, while the unique quadratic growth (QG) condition guarantees that the global optimal solution of the function $g$ is uniquely determined. The well-known QG condition \citep{Ioffe1994OnS, Bonnans1995Sec}, defined as $g(\bm{x}) - g^\star \geq \frac{\mu}{2}\| \bm{x} - \bm{x}_p \|^2$, where $\bm{x}_p$ denotes the projection of $\bm{x}$ onto the solution set, does not guarantee the uniqueness of the solution. We refer to this QG condition that holds when the solution set contains exactly one element as the unique QG condition. These properties are known to be contained within strong convexity (SC), and (SC)$\Rightarrow$(PL) and (SC)$\Rightarrow$(unique QG) hold. Note that under $L_g$-smoothness, (PL)$\Rightarrow$(QG) holds \citep[Theorem 2]{Karimi2016Lin}, but (PL)$\Rightarrow$(unique QG) does not hold. Combining PL and unique QG satisfies the aforementioned condition (1). Finally, to satisfy condition (2), we assume $L_g$-smoothness. That is, we can relax the $\sigma$-nice property as follows.
\begin{dfn}
[weakly $(\mu_1, \mu_2)$-nice function]\label{dfn:weak} Let $M \in \mathbb{N}$, $m \in [M]$, $\gamma \in [0.5, 1)$, and $\delta_{m+1} := \gamma\delta_m$. A function $f \colon \mathbb{R}^d \to \mathbb{R}$ is said to be weakly $(\mu_1, \mu_2)$-nice if the following two conditions hold:

{\em (i)} For all $\delta_m > 0$ and all $\bm{x}_{\delta_m}^\star$, there exists $\bm{x}_{\delta_{m+1}}^\star$ such that: $\left\| \bm{x}_{\delta_m}^\star - \bm{x}_{\delta_{m+1}}^\star \right\| \leq (1-\gamma)\delta_m$.

{\em (ii)} For all $\delta_m>0$, the function $\hat{f}_{\delta_m}(\bm{x})$ is $\mu_1$-PL, unique $\mu_2$-QG, and $L_g$-smooth over $B(\bm{x}_{\delta_m}^\star;3\delta_m)$.
\end{dfn}
The weakly $(\mu_1, \mu_2)$-nice property (i) is a slight extension of the $\sigma$-nice property (i), and it coincides with the $\sigma$-nice property (i) when the decay rate $\gamma$ of the smoothing degree is 0.5. The $\sigma$-nice property (ii) does not require $L_g$-smoothness, whereas the weakly $(\mu_1, \mu_2)$-nice property (ii) does require it. Therefore, this extension is not so large. However, the weakly $(\mu_1, \mu_2)$-nice property does not require the local (strong) convexity of the smoothed function and permits non-convexity as long as the stationary point is unique. Since (SC)$\Rightarrow$(PL) and (SC)$\Rightarrow$(unique QG) hold, the weakly $(\mu_1, \mu_2)$-nice property (ii) extracts the property that graduated optimization is truly required within the $\sigma$-nice property (ii). It can be said to better capture the power of graduated optimization, which represents a significant relaxation of conditions. Note that as long as the original function $f$ is $L_g$-smooth (Assumption \ref{assum:01}), the $L_g$-smoothness of the smoothed function $\hat{f}_{\delta}$ naturally follows (see Lemma \ref{lem:04}).

\begin{wrapfigure}{r}{60mm}
\vspace*{-8pt}
  \centering
  \includegraphics[width=58mm]{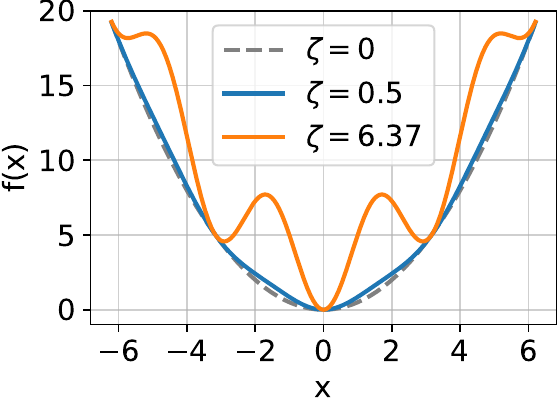} 
  \vspace*{-8pt}
    \caption{1D illustrative plot.}
  \label{fig:nice}
\end{wrapfigure}
Finally, let us introduce an example of a weakly $(\mu_1, \mu_2)$-nice function. The function $f(\bm{x}) := \frac{1}{2}\|\bm{x}\|^2 + \zeta \sum_{i=1}^{d} \sin^2(x_i)$ $(\bm{x} := (x_1, \ldots, x_d)^\top)$ is a quadratic function with added sine noise, where $\zeta>0$ controls the noise magnitude. Increasing $\zeta$ adds more valleys and deepens them. If Gaussian noise is used for smoothing, this function is $\sigma$-nice when $\zeta < 0.5$ and a weakly $(\mu_1, \mu_2)$-nice function when $\zeta < 6.374\cdots$ (see Appendix \ref{sec:nice}). Figure \ref{fig:nice} plots the one-dimensional version of this function. We can see that the weakly $(\mu_1, \mu_2)$-nice property allows for more non-convexity in the function than does the $\sigma$-nice property.

\subsection{Analysis of Implicit Graduated Optimization}\label{sec:4.2}
In this section, we construct an implicit graduated optimization algorithm that varies the learning rate $\eta$ and batch size $b$ so that the degree of smoothing $\delta = \frac{\eta C_\alpha^{1/\alpha} \tau}{b^{(\alpha-1)/\alpha}}$ by stochastic noise in SGD gradually decreases and then analyze its convergence. 
Algorithm \ref{alg:gnc2} embodies the framework of implicit graduated optimization with SGD for weakly $(\mu_1, \mu_2)$-nice functions, while Algorithm \ref{alg:sgd2} is used to optimize each smoothed function; it must be GD wth noise (see equation (\ref{eq:27})). 

\begin{figure}[t]
\centering
\begin{minipage}[t]{0.54\linewidth}
\begin{algorithm}[H]
\caption{Implicit Graduated Optimization}
\label{alg:gnc2}
\begin{algorithmic}
\REQUIRE{$\epsilon \in (0,1), \bm{x}_1 \in B(\bm{x}_{\delta_1}^\star; 3\delta_1), b_1 \in [n],$}\\%
\qquad\quad \ $\eta_1>0, \gamma \in [0.5,1), \alpha \in (1,2]$
\STATE{$\delta_1 = \frac{\eta_1 C_\alpha^{1/\alpha}\tau}{b_1^{(\alpha-1)/\alpha}}, \beta_0 \leq \min\left\{ \frac{\gamma}{4L_f\delta_1}, \frac{\gamma}{\sqrt{\mu_2}\delta_1} \right\},$}%
\STATE{$M= \log_{\gamma}{\beta_0 \epsilon}$}
\FOR{$m=1$ to $M$}
\STATE{$\epsilon_m := \mu_2\delta_m^2/2, \ T_m := H_m / (\epsilon_m - I_m)$}
\STATE{$\kappa_m / \lambda_m^{(\alpha-1)/\alpha} = \gamma \ (\kappa_m \in (0,1], \lambda_m \geq 1)$}
\STATE{$\bm{x}_{m+1} := \text{GD with noise}(T_m, \bm{x}_m, \hat{f}_{\delta_m}, \eta_m)$}
\STATE{$\eta_{m+1} := \kappa_m \eta_m, b_{m+1} := \lambda_m b_m$}
\STATE{$\delta_{m+1} := \frac{\eta_{m+1} C_{\alpha}^{1/\alpha}\tau}{b_{m+1}^{(\alpha-1)/\alpha}}$}
\ENDFOR
\RETURN $\bm{x}_{M +1}$
\end{algorithmic}
\end{algorithm}
\end{minipage}
\hfill
\begin{minipage}[t]{0.45\linewidth}
\begin{algorithm}[H]
\caption{GD with noise}
\label{alg:sgd2}
\begin{algorithmic}
\REQUIRE$T_m, \hat{\bm{x}}_1^{(m)}, \hat{f}_{\delta_m}, \eta_m > 0$
\STATE{$\bm{\nu}_t^{(m)}$ is zero mean noise for analysis.}
\FOR{$t=1$ to $T_m$}
\STATE{$\hat{\bm{x}}_{t+1}^{(m)} := \hat{\bm{x}}_t^{(m)} - \eta_m \nabla \hat{f}_{\delta_m} (\hat{\bm{x}}_t^{(m)}) + \bm{\nu}_t^{(m)}$}
\ENDFOR
\RETURN $\hat{\bm{x}}_{T_m +1}^{(m)}$
\end{algorithmic}
\end{algorithm}
\end{minipage}
\end{figure}

\begin{rem}
Note that our implicit graduated optimization (Algorithm \ref{alg:gnc2}) is achieved by SGD with a decaying learning rate and/or increasing batch size. Given the powerful fact (equation (\ref{eq:27})) that SGD for a function $f$ is equivalent to GD with noise for a function $\hat{f}_{\delta_m}$, the analysis can be performed for GD with noise. The operation of smoothing $f$ to obtain $\hat{f}_{\delta_m}$ is implicitly achieved through the stochastic noise of SGD, rather than by using integration or an approximation. Therefore, gradient descent is not actually executed, the gradient of $\hat{f}_{\delta_m}$ is not computed, and only the stochastic gradient of $f$ is required. This is why our proposed method is ``implicit'' graduated optimization. Furthermore, the parameters $L_f$ and $\tau$ appearing in Algorithm \ref{alg:gnc2} are merely tools for analysis and do not need to be known beforehand. In practice, only the learning rate and batch size need to be specified to execute SGD.
\end{rem}

From the definition of weakly $(\mu_1, \mu_2)$-nice functions, the smoothed function $\hat{f}_{\delta_m}$ has the $L_g$-smoothness, $\mu_1$-PL, and $\mu_2$-QG in $B(\bm{x}_{\delta_m}^\star ; 3\delta_m)$ proprties. Also, the learning rate used by Algorithm \ref{alg:sgd2} to optimize $\hat{f}_{\delta_m}$ should always be constant. Therefore, let us examine the convergence of GD with noise with a constant learning rate for $\hat{f}_{\delta_m}$ that is an $L_g$-smooth, $\mu_1$-PL, and unique $\mu_2$-QG function. The proof of Theorem \ref{thm:03} is in Appendix \ref{subsec:9.3}.

\begin{thm}
[Convergence analysis of Algorithm \ref{alg:sgd2}]\label{thm:03} Let $f \colon \mathbb{R}^d \to \mathbb{R}$ be an $L_f$-Lipschitz weakly $(\mu_1, \mu_2)$-nice function and $K:=L_g^{\alpha-1}(2L_f)^{2-\alpha}$. Suppose that $\eta_m < \min\left\{ \frac{1}{\mu_1}, \left( \frac{2}{K} \right)^{\frac{1}{\alpha-1}}\right\}$ and $\hat{\bm{x}}_1^{(m)} \in B(\bm{x}_{\delta_m}^\star; 3\delta_m)$ for all $m \in [M]$. Then, the sequence $(\hat{\bm{x}}_t^{(m)})_{t \in \mathbb{N}}$ generated by Algorithm \ref{alg:sgd2} satisfies
\begin{align}
\min_{t \in [T]} \mathbb{E}\left[ \hat{f}_{\delta_m} \left(\hat{\bm{x}}_t^{(m)} \right) \right] - \hat{f}_{\delta_m}(\bm{x}_{\delta_m}^\star)
\leq \frac{H_m}{T} + I_m
= \mathcal{O}\left( \frac{1}{T} + \delta_m^\alpha \right), \label{eq:29}%
\end{align}
where $H_m := \frac{2L_f^2}{\mu_1\mu_2\left( 2-K\eta_m^{\alpha-1}\right)}$ and $I_m := \frac{K\left\{ (2-\alpha)\eta_m^\alpha + 2^{\alpha}(1+2^\alpha L_g^\alpha) \delta_m^\alpha \right\}}{\alpha\mu_1\mu_2\left( 2-K\eta_m^{\alpha-1}\right)}$ are nonnegative constants.
\end{thm}

Theorem \ref{thm:03} shows that Algorithm \ref{alg:sgd2} can reach an $\epsilon_m$-neighborhood of the optimal solution $\bm{x}_{\delta_m}^\star$ of $\hat{f}_{\delta_m}$ in approximately $T_m := H_m / (\epsilon_m - I_m)$ iterations. 
Theorem \ref{thm:03} requires that the initial point $\hat{\bm{x}}_1^{(m)}$ at each stage be contained within a local favorable region $B(\bm{x}_{\delta_m}^\star; 3\delta_m)$ of $\hat{f}_{\delta_m}$, and the following proposition guarantees this. The proof of Proposition \ref{prop:999} is in Appendix \ref{sec:E.3}.

\begin{prop}\label{prop:999}
Let $f$ be a weakly $(\mu_1, \mu_2)$-nice function and $\delta_{m+1} := \gamma \delta_m$. Suppose that $\gamma \in [0.5,1)$ and $\bm{x}_1 \in B(\bm{x}_{\delta_1}^\star; 3\delta_1)$. Then for all $m \in [M]$, $\mathbb{E}\left[ \left\| \bm{x}_m - \bm{x}_{\delta_m}^\star \right\| \right] < 3 \delta_m.$
\end{prop}
$\bm{x}_m$ is the approximate solution obtained by optimization of the smoothed function $\hat{f}_{\delta_{m-1}}$ with Algorithm \ref{alg:sgd2} and is the initial point of optimization of the next smoothed function $\hat{f}_{\delta_m}$. Therefore, Proposition \ref{prop:999} implies that $\gamma \in [0.5,1)$ must hold for the initial point of optimization of $\hat{f}_{\delta_m}$ to be contained in the local favorable region of $\hat{f}_{\delta_m}$. Therefore, from Theorem \ref{thm:03} and Proposition \ref{prop:999}, if $f$ is a weakly $(\mu_1, \mu_2)$-nice function and $\bm{x}_1 \in B(\bm{x}_{\delta_1}^\star ; 3\delta_1)$ holds, the sequence $(\bm{x}_m)_{m \in [M]}$ generated by Algorithm \ref{alg:gnc2} never goes outside of the local favorable region $B(\bm{x}_{\delta_m}^\star; 3\delta_m)$ of each smoothed function $\hat{f}_{\delta_m}$ $(m \in [M])$.

The next theorem guarantees the convergence of Algorithm \ref{alg:gnc2} with the weakly $(\mu_1, \mu_2)$-nice function (The proof of Theorem \ref{thm:3.4} is in Appendix \ref{subsec:9.4}). 
\begin{thm}
[Convergence analysis of Algorithm \ref{alg:gnc2}]\label{thm:3.4} Let $\epsilon \in (0,1)$, and $f \colon \mathbb{R}^d \to \mathbb{R}$ be an $L_f$-Lipschitz weakly $(\mu_1, \mu_2)$-nice function. Suppose that we run Algorithm \ref{alg:gnc2}; then after $\Omega\left( 1/\epsilon^{2}\right)$ rounds, the algorithm reaches an $\epsilon$-neighborhood of the global optimal solution $\bm{x}^\star$.
\end{thm}

\begin{rem}
Hazan et al. proved that, for a $\sigma$-nice function, an explicit graduated optimization algorithm employing SGD with a decaying learning rate for optimizing each smoothed functions can reach a global optimal solution in $\mathcal{O}\left(1/\epsilon^2 \right)$ iterations \citep{Elad2016OnG}. In Theorem \ref{thm:03}, adopting an appropriate decaying learning rate improves the convergence rate to $\mathcal{O}(1/T)$, and consequently, the computational complexity in Theorem \ref{thm:3.4} also improves to $\mathcal{O}\left(1/\epsilon^2 \right)$. Therefore, the explicit graduated optimization algorithm using a decaying learning rate instead of a constant learning rate in Algorithm \ref{alg:sgd2} converges to a global optimal solution in $\mathcal{O}(1/\epsilon^2)$ iterations for weakly $(\mu_1, \mu_2)$-nice functions, just as it does for $\sigma$-nice functions.
Although we cannot adopt a decreasing learning rate in Algorithm \ref{alg:sgd2} due to our focus on implicit graduated optimization, the above fact is crucial for guaranteeing the effectiveness of weakly $(\mu_1, \mu_2)$-nice functions.
\end{rem}

\section{Conclusion}
\label{sec:con}
We extended the definition of function smoothing from Gaussian noise to heavy-tailed noise. This enabled us to evaluate the implicit smoothing effect of noise on the objective function even when the stochastic noise in SGD follows heavy-tailed noise. Since the cross-entropy loss defined using the output of a DNN equipped with some weight regularization technique can be smoothed even by heavy-tailed noise, our smoothing theory has very broad applicability. Furthermore, we proposed weakly $(\mu_1,\mu_2)$-nice functions, a broader family of nonconvex functions than $\sigma$-nice functions, and proved that implicit graduated optimization algorithms exploiting the smoothing effect of SGD's light- or heavy- tailed stochastic noise converge to the global optimal solution of weakly $(\mu_1,\mu_2)$-nice functions. 



\bibliography{cvpr2024}

\appendix


\section{Derivation of equation (\ref{eq:1})}
\label{sec:der}
Let $\bm{y}_{t}$ be the parameter updated by gradient descent (GD) and $\bm{x}_{t+1}$ be the parameter updated by SGD at time $t$, i.e., 
\begin{align*}
\bm{y}_{t} &:= \bm{x}_t - \eta \nabla f(\bm{x}_t), \\
\bm{x}_{t+1} &:= \bm{x}_t - \eta \nabla f_{\mathcal{S}_t}(\bm{x}_t) \\
 &= \bm{x}_t - \eta (\nabla f(\bm{x}_t) + \bm{\omega}_t).
\end{align*}
Then, we have
\begin{align}
\bm{x}_{t+1} 
&:= \bm{x}_t - \eta \nabla f_{\mathcal{S}_t}(\bm{x}_t) \nonumber \\
&= \left( \bm{y}_t + \eta \nabla f(\bm{x}_t) \right) - \eta \nabla f_{\mathcal{S}_t}(\bm{x}_t) \nonumber \\
&= \bm{y}_t - \eta \bm{\omega}_t, \label{eq:35}
\end{align}
from $\bm{\omega}_t := \nabla f_{\mathcal{S}_t} (\bm{x}_t) - \nabla f (\bm{x}_t)$. Hence,
\begin{align*}
\bm{y}_{t+1} 
&= \bm{x}_{t+1} - \eta \nabla f(\bm{x}_{t+1}) \\
&= \bm{y}_t - \eta \bm{\omega}_t - \eta \nabla f(\bm{y}_t - \eta \bm{\omega}_t).
\end{align*}
By taking the expectation with respect to $\bm{\omega}_t$ on both sides, we have, from $\mathbb{E}_{\bm{\omega}_t} \left[ \bm{\omega}_t \right] = \bm{0}$,
\begin{align*}
\mathbb{E}_{\bm{\omega}_t} \left[ \bm{y}_{t+1} \right]
=\mathbb{E}_{\bm{\omega}_t} \left[ \bm{y}_t \right] - \eta \nabla \mathbb{E}_{\bm{\omega}_t} \left[f(\bm{y}_t - \eta \bm{\omega}_t) \right],
\end{align*}
where we have used $\mathbb{E}_{\bm{\omega}_t} \left[ \nabla f(\bm{y}_t - \eta \bm{\omega}_t) \right] = \nabla \mathbb{E}_{\bm{\omega}_t} \left[f(\bm{y}_t - \eta \bm{\omega}_t) \right]$, which holds for a Lipschitz-continuous and differentiable $f$ \citep[Theorem 7.49]{Shapiro2009Lec}. In addition, from (\ref{eq:35}) and $\mathbb{E}_{\bm{\omega}_t} \left[\bm{\omega}_t \right] = \bm{0}$, we obtain
\begin{align*}
\mathbb{E}_{\bm{\omega}_t} \left[ \bm{x}_{t+1} \right]
= \bm{y}_t.
\end{align*}
Therefore, on average, the parameter $\bm{x}_{t+1}$ of the function $f$ arrived at by SGD coincides with the parameter $\bm{y}_t$ of the smoothed function $\hat{f}(\bm{y}_t) := \mathbb{E}_{\bm{\omega}_t} \left[f(\bm{y}_t - \eta \bm{\omega}_t) \right]$ arrived at by GD.

\section{Proofs of the Lemmas in main body}
\label{sec:C}
\subsection{Proof of Lemma \ref{lem:04}}
\begin{proof}
From Definition \ref{dfn:newfhat} and Assumption \ref{assum:01}, we have, for all $\bm{x}, \bm{y} \in \mathbb{R}^d$,
\begin{align*}
\left\| \nabla \hat{f}_\delta(\bm{x}) - \nabla \hat{f}_\delta(\bm{y}) \right\|
&=\left\| \nabla \mathbb{E}_{\bm{u}}\left[f(\bm{x}-\delta \bm{u})\right] - \nabla \mathbb{E}_{\bm{u}}\left[f(\bm{y}-\delta \bm{u})\right] \right\| \\
&=\left\| \mathbb{E}_{\bm{u}} \left[\nabla f(\bm{x} - \delta \bm{u})\right] - \mathbb{E}_{\bm{u}} \left[\nabla f(\bm{y} - \delta \bm{u})\right] \right\| \\
&=\left\| \mathbb{E}_{\bm{u}} \left[\nabla f(\bm{x} - \delta \bm{u}) - \nabla f(\bm{y} - \delta \bm{u})\right] \right\| \\
&\leq \mathbb{E}_{\bm{u}} \left[ \left\| \nabla f(\bm{x} - \delta \bm{u}) - \nabla f(\bm{y} - \delta \bm{u}) \right\| \right] \\
&\leq \mathbb{E}_{\bm{u}} \left[L_g \left\| (\bm{x} - \delta \bm{u}) - (\bm{y} - \delta \bm{u}) \right\| \right] \\
&= \mathbb{E}_{\bm{u}} \left[L_g \left\| \bm{x} - \bm{y} \right\| \right] \\
&= L_g \| \bm{x} - \bm{y} \|.
\end{align*}
In addition, from Assumption \ref{assum:02}, we have, for all $\bm{x}, \bm{y} \in \mathbb{R}^d$,
\begin{align*}
\left|\hat{f}_\delta(\bm{x}) - \hat{f}_\delta(\bm{y}) \right|
&=\left| \mathbb{E}_{\bm{u}}\left[f(\bm{x}-\delta \bm{u})\right] - \mathbb{E}_{\bm{u}}\left[f(\bm{y}-\delta \bm{u})\right] \right| \\
&=\left| \mathbb{E}_{\bm{u}} \left[f(\bm{x} - \delta \bm{u}) - f(\bm{y} - \delta \bm{u})\right] \right| \\
&\leq \mathbb{E}_{\bm{u}} \left[\left| f(\bm{x} - \delta \bm{u}) - f(\bm{y} - \delta \bm{u}) \right|\right] \\
&\leq \mathbb{E}_{\bm{u}} \left[L_f \| (\bm{x} - \delta \bm{u}) - (\bm{y} - \delta \bm{u}) \|\right] \\
&= \mathbb{E}_{\bm{u}} \left[L_f \left\| \bm{x} - \bm{y} \right\| \right] \\
&= L_f \| \bm{x} - \bm{y} \|.
\end{align*}
This completes the proof.
\end{proof}

\subsection{Lemma for proof of Lemma \ref{lem:00}}
\begin{lem}
[von Bahr-Esseen inequality \citep{Bahr1965Ine}]\label{lem:vBE} Let $\alpha \in [1,2]$, and let $X_1, X_2, \ldots, X_n$ be a sequence of independent random variables with $\mathbb{E}[X_i] = 0$ and $\mathbb{E}\left[ |X_i|^\alpha \right] < \infty$ for all $i \in [n]$. Then, 
\begin{align*}
\mathbb{E}\left[ \left| \sum_{i \in [n]} X_i \right|^\alpha \right]
\leq C_\alpha \sum_{i\in[n]}\mathbb{E}\left[ |X_i|^\alpha \right],
\end{align*}
where $C_\alpha$ is a constant that depends only on $\alpha$. According to \citep[Proposition 1.8]{Pinelis2015Bes}, $C_\alpha$ is continuously and strictly decreasing in $\alpha \in (1,2]$ from $\displaystyle \lim_{\alpha \to1} C_\alpha = 2$ to $C_2 = 1$.
\end{lem}

\subsection{Proof of Lemma \ref{lem:00}}\label{sec:C.3}
\begin{proof}
From the Assumptions \ref{assum:03} and Lemma \ref{lem:vBE}, we have
\begin{align*}
\mathbb{E}_{\bm{\xi}_t} \left[ \left\| \nabla f_{\mathcal{S}_t}(\bm{x}_t) - \nabla f(\bm{x}_t) \right\|^\alpha \right]
&= \mathbb{E}_{\bm{\xi}_t} \left[ \left\| \frac{1}{b}\sum_{i \in [b]} \mathsf{G}_{\xi_{t,i}}(\bm{x}_t) - \nabla f(\bm{x}_t) \right\|^\alpha \right] \\
&= \mathbb{E}_{\bm{\xi}_t} \left[ \left\| \frac{1}{b}\sum_{i \in [b]} \left(\mathsf{G}_{\xi_{t,i}}(\bm{x}_t) - \nabla f(\bm{x}_t) \right) \right\|^\alpha \right] \\
&= \frac{1}{b^\alpha} \mathbb{E}_{\bm{\xi}_t} \left[ \left\| \sum_{i \in [b]} \left(\mathsf{G}_{\xi_{t,i}}(\bm{x}_t) - \nabla f(\bm{x}_t) \right) \right\|^\alpha \right] \\
&\leq \frac{C_\alpha}{b^\alpha} \sum_{i\in[b]} \mathbb{E}_{\bm{\xi}_t}\left[ \| \mathsf{G}_{\xi_{t,i}}(\bm{x}_t) - \nabla f(\bm{x}_t) \|^\alpha \right] \\
&= \frac{C_\alpha \tau^\alpha}{b^{\alpha-1}}.
\end{align*}
This completes the proof.
\end{proof}

\UPDATE{
\section{Discussion on Growth Order of the Empirical Loss Function}
\label{sec:homo}
A neural network $\Phi$ is said to be homogenous if there is a number $L > 0$ such that the network output $\Phi(\bm{x}; \bm{z})$, where $\bm{x}$ stands for the parameter and $\bm{z}$ stands for the input, satisfies the following condition:
\begin{align*}
\forall c >0 \colon \Phi(c\bm{x}; \bm{z}) = c^L \Phi(\bm{x}; \bm{z}) \ \ \text{for all $\bm{x}$ and $\bm{z}$}.
\end{align*}
Many neural networks are homogenous \citep{Neyshabur2015Pat, Du2018Alg}, especially, any deep fully connected neural network or deep CNN with ReLU or LeakyReLU activation has an order $L$ equal to the number of layers when all bias terms are removed \citep{Lyu2020Gra}. Since any parameter $\bm{x}$ can be expressed as $\bm{x}=c \bm{s}$ $(c \geq 0, \| x \| = c, \| \bm{s} \| = 1)$, under homogenous, the following holds:
\begin{align*}
\| \Phi(\bm{x}) \| = \| \Phi(c \bm{s}) \| = c^L \| \Phi(\bm{s}) \| = \| x \|^L \| \Phi(\bm{s})\| \leq C \| \bm{x} \|^L,
\end{align*}
where we assume $\| \Phi(\bm{s})\| \leq C$ for $C > 0$. If the correct label is a one-hot vector and the correct index is $*$, then the cross-entropy loss can be expressed as follows:
\begin{align*}
f(\bm{x}) 
&= -\log \left( \frac{e^{\Phi(\bm{x})_*}}{\sum_{k \in [K]}e^{\Phi(\bm{x})_k}}\right)
=\log\left( \sum_{k \in [K]} e^{\Phi(\bm{x})_k}\right) - \Phi(\bm{x})_* \\
&\leq \max_{k \in [K]} \Phi(\bm{x})_k + \log K - \Phi(\bm{x})_* 
\leq \left|\max_{k \in [K]} \Phi(\bm{x})_k \right| + \log K + \left|\Phi(\bm{x})_* \right| \\
&\leq 2\| \Phi(\bm{x})\| + \log K \leq 2C\| \bm{x} \|^L + \log K
= \mathcal{O}\left( \| \bm{x} \|^L \right),
\end{align*}
where $K$ is the number of classes. On the other hand, according to \citep{Wan2021Sph}, when incorporating the commonly used normalization unit into the model, the parameter $\bm{x}$ exhibits scale-invariance with respect to the model output $\Phi(\bm{x})$, i.e., 
\begin{align*}
\forall c>0 \colon \Phi(c\bm{x}; \bm{z}) = \Phi(\bm{x}; \bm{z}) = c^0 \Phi(\bm{x}; \bm{z}) \ \ \text{for all $\bm{x}$ and $\bm{z}$}.%
\end{align*}
A similar argument to the one above shows that $f(\bm{x}) = \mathcal{O}\left( \| \bm{x} \|^0 \right)$.

\section{Example of $\sigma$-Nice Function and Weakly $(\mu_1, \mu_2)$-Nice Function}
\label{sec:nice}
Let us consider the following function:
\begin{align*}
f(\bm{x}) := \frac{1}{2}\|\bm{x}\|^2 + \zeta \sum_{i=1}^{d} \sin^2(x_i),
\end{align*}
where $\bm{x} := (x_1, \ldots, x_d)^\top, \zeta>0$. First, we need to derive a smoothed version $\hat{f}_{\delta}$.
\begin{align*}
\hat{f}_{\delta}(\bm{x})
&= \mathbb{E}_{\bm{u} \sim \mathcal{N}\left( \bm{0}; \frac{1}{d}I_d \right)}\left[ f(\bm{x} - \delta \bm{u})\right] \\
&= \mathbb{E}_{\bm{u} \sim \mathcal{N}\left( \bm{0}; \frac{1}{d}I_d \right)}\left[ \frac{1}{2}\| \bm{x} - \delta \bm{u} \|^2 + \zeta \sum_{i=1}^{d} \sin^2(x_i - \delta u_i) \right] \\
&= \mathbb{E}_{\bm{u} \sim \mathcal{N}\left( \bm{0}; \frac{1}{d}I_d \right)} \left[ \frac{1}{2}\|\bm{x}\|^2 - \delta\langle \bm{x},\bm{u} \rangle + \frac{\delta^2}{2} \| \bm{u} \|^2 \right] + \zeta\sum_{i=1}^{d} \mathbb{E}_{u_i \sim \mathcal{N}\left( 0; \frac{1}{\sqrt{d}} \right)}\left[ \sin^2(x_i - \delta u_i) \right] \\
&=\frac{1}{2}\| \bm{x} \|^2 + \frac{\delta^2}{2} + \zeta\sum_{i=1}^{d} \mathbb{E}_{u_i \sim \mathcal{N}\left( 0; \frac{1}{\sqrt{d}} \right)}\left[ \sin^2(x_i - \delta u_i) \right].
\end{align*}
where we use $\mathbb{E}_{\bm{u}}[\bm{u}] = \bm{0}$ and $\mathbb{E}_{\bm{u}}[\| \bm{u} \|^2] = 1$. Here, we have
\begin{align*}
\mathbb{E}_{u_i \sim \mathcal{N}\left( 0; \frac{1}{\sqrt{d}} \right)}\left[ \sin^2(x_i - \delta u_i) \right]
&= \frac{1}{2} \mathbb{E}_{u_i \sim \mathcal{N}\left( 0; \frac{1}{\sqrt{d}} \right)} \left[ 1-\cos(2x_i - 2\delta u_i))\right] \\
&= \frac{1}{2} - \frac{1}{2}\mathbb{E}_{u_i \sim \mathcal{N}\left( 0; \frac{1}{\sqrt{d}} \right)} \left[ \cos(2x_i)\cos(2\delta u_i) + \sin(2x_i)\sin(2\delta u_i)\right] \\
&= \frac{1}{2} - \frac{1}{2}e^{-\frac{2\delta^2}{d}}\cos(2x_i).
\end{align*}
Therefore, we have
\begin{align}\label{eq:func}
\hat{f}_{\delta}(\bm{x}) = \frac{1}{2}\| \bm{x} \|^2 - \frac{\zeta}{2}e^{-\frac{2\delta^2}{d}} \sum_{i=1}^{d} \cos(2x_i) + \frac{\delta^2}{2} + \frac{\zeta d}{2}.
\end{align}
Hence, its gradient and Hessian can be derived as follows:
\begin{align}
&\nabla \hat{f}_{\delta}(\bm{x}) = \bm{x} + \zeta e^{-\frac{2\delta^2}{d}}
\begin{pmatrix}
 \sin(2x_1) \\
 \vdots \\
 \sin(2x_d)
\end{pmatrix},\label{eq:grad} \\ 
&\nabla^2 \hat{f}_{\delta}(\bm{x}) = 
\begin{pmatrix}
1+2\zeta e^{-\frac{2\delta^2}{d}}\cos(2x_1) & & 0 \\
 & \ddots & \\
 0 & & 1+2\zeta e^{-\frac{2\delta^2}{d}}\cos(2x_d) \label{eq:hesse}
\end{pmatrix}.
\end{align}

\paragraph{Weakly $(\mu_1, \mu_2)$-nice property (i)} For any $\delta > 0$, $\hat{f}_{\delta}$ has a global minimum at $\bm{x}=\bm{0}$, so $\bm{x}_{\delta}^\star = \bm{0}$ holds for all $\delta>0$. Thus, for all $\delta_m > 0$ and all $\gamma \in [0.5,1)$, 
\begin{align*}
\| \bm{x}_{\delta_m}^\star - \bm{x}_{\delta_{m+1}}^\star \|
= \| \bm{0} - \bm{0} \|
= 0
\leq (1-\gamma) \delta_m,
\end{align*} 
which implies that the weakly $(\mu_1, \mu_2)$-nice property (i) holds.

\paragraph{Weakly $(\mu_1, \mu_2)$-nice property (ii)} First, let us confirm the smoothness of $\hat{f}_{\delta_m}$. From Equation \eqref{eq:hesse}, for all $\delta>0$,
\begin{align*}
\left\| \nabla^2 \hat{f}_{\delta}(\bm{x}) \right\|_2 
= \max_{i \in [d]} \left\{ 1 + 2\zeta e^{-\frac{2\delta^2}{d}}\cos(2x_i) \right\}
\leq 1 + 2\zeta,
\end{align*}
where we use $e^{-\frac{2\delta^2}{d}} \leq 1$ and $\cos(2x_i) \leq 1$. Therefore, for all $\delta_m > 0$, $\hat{f}_{\delta_m}$ is $(1+2\zeta)$-smooth. 
\noindent
Next, let us examine the unique $\mu_2$-QG condition: $\hat{f}_{\delta}(\bm{x}) - \hat{f}_{\delta}(\bm{x}_{\delta}^\star) \geq \frac{\mu_2}{2}\| \bm{x} - \bm{x}_{\delta}^\star \|^2$ $(\mu_2>0)$. From equation \eqref{eq:func}, for all $\delta>0$, 
\begin{align}\label{eq:funcd}
\hat{f}_{\delta}(\bm{x}) - \hat{f}_{\delta}(\bm{x}_\delta^\star)
= \frac{1}{2}\| \bm{x} \|^2 + \frac{\zeta}{2}e^{-\frac{2\delta^2}{d}} \sum_{i=1}^{d} \left( 1 - \cos(2x_i) \right)
= \frac{1}{2}\| \bm{x} \|^2 + \zeta e^{-\frac{2\delta^2}{d}} \sum_{i=1}^{d} \sin^2(x_i).%
\end{align}
On the other hand, from $\bm{x}_{\delta}^\star = \bm{0}$, for all $\delta>0$,
\begin{align*}
\frac{1}{2}\| \bm{x} - \bm{x}_{\delta}^\star \|^2 = \frac{1}{2}\| \bm{x} \|^2.
\end{align*}
Hence, for all $\delta_m > 0$, $\hat{f}_{\delta_m}$ satisfies the unique $\mu_2$-QG condition, where $\mu_2 \in (0,1]$.
\noindent
Finally, let us examine the $\mu_1$-PL condition: $\frac{1}{2}\| \nabla \hat{f}_{\delta}(\bm{x}) \|^2 \geq \mu_1(\hat{f}_{\delta}(\bm{x}) - \hat{f}_{\delta}(\bm{x}_{\delta}^\star))$ $(\mu_1 > 0)$. From equation \eqref{eq:grad}, for all $\delta>0$, 
\begin{align*}
\frac{1}{2}\| \nabla \hat{f}_{\delta}(\bm{x}) \|^2 
= \frac{1}{2}\sum_{i=1}^{d} \left( x_i + \zeta e^{-\frac{2\delta^2}{d}} \sin(2x_i) \right)^2
\geq \frac{1}{2}\sum_{i=1}^{d} \left( x_i + \zeta e^{-2\delta^2} \sin(2x_i) \right)^2.
\end{align*}
On the other hand, from equation \eqref{eq:funcd}, for all $\delta>0$,
\begin{align*}
\hat{f}_{\delta}(\bm{x}) - \hat{f}_{\delta}(\bm{x}_\delta^\star)
= \sum_{i=1}^{d} \left( \frac{1}{2}x_i^2 + \zeta e^{-\frac{2\delta^2}{d}} \sin^2(x_i) \right)
\leq \sum_{i=1}^{d} \left( \frac{1}{2}x_i^2 + \zeta \right)
\end{align*}
For all $x \in B(0;3\delta) \setminus \{0\}$, let us consider the following function:
\begin{align*}
g_1(x) := \frac{(x + \zeta e^{-2\delta^2}\sin(2x))^2}{x^2 + 2\zeta}.
\end{align*}
For the $\mu_1$-PL condition to hold, it suffices that $g(x)$ has a positive lower bound for any $\delta>0$. From $x \in B(0;3\delta) \setminus \{ 0\}$, we have
\begin{align*}
g_1(x) \geq \frac{(x + \zeta e^{-2\delta^2}\sin(2x))^2}{9\delta^2 + 2\zeta} =: \frac{(g_2(x))^2}{9\delta^2 + 2\zeta}.
\end{align*}
When $g_2(x)$ has no solutions other than $x = 0$, $g_1(x)$ clearly has a positive lower bound. Consider the case where $g_2(x) = 0$ has solutions other than $x = 0$. Let $x^*$ denote the smallest positive solution among these. If $x^*$ satisfies $x^* \not\in B(0;3\delta)$, then $g_1(x)$ has a positive lower bound with $x \in B(0;3\delta)$. We wish to find the largest possible $\zeta>0$ such that $x^*$ satisfies $x^* \not\in B(0;3\delta)$ for any given $\delta>0$. The condition $x^* \not\in B(0;3\delta)$ ceases to hold precisely when $x^* = 3\delta$. Therefore, we need only consider $g_2(3\delta)=3\delta + \zeta e^{-2\delta^2}\sin(6\delta) = 0$, i.e., 
\begin{align*}
\zeta(\delta) = \frac{-3\delta}{e^{-2\delta^2}\sin(6\delta)}\ \ \left(\frac{\pi}{6} < \delta < \frac{\pi}{3} \right).
\end{align*}
That is, for any given $\delta$, if $\zeta < \zeta(\delta)$, then either $x^*$ does not exist or $x^* \not\in B(0;3\delta)$ holds. If $\zeta > \zeta(\delta)$, then $x^* \in B(0;3\delta)$ holds. If $\zeta = \zeta(\delta)$, then $x^* = 3\delta$, i.e., $x^* \in B(0;3\delta)$ holds. Numerical calculations indicate that $\zeta(\delta)$ has a minimum value of $6.374\cdots$, so setting $\zeta < 6.374\cdots$ ensures that $g_1(x)$ possesses a positive lower bound for any $\delta>0$. Consequently, the $\mu_1$-PL condition is satisfied. From the above, if $\zeta < 6.374\cdots$, then $f$ is a weakly $(\mu_1, \mu_2)$-nice function.

\paragraph{$\sigma$-nice property} Similarly to the discussion on the weakly $(\mu_1, \mu_2)$-nice property (i), the $\sigma$-nice property (i) is satisfied. Let us consider the $\sigma$-nice property (ii). From equation \eqref{eq:hesse}, 
\begin{align*}
\left\| \nabla^2 \hat{f}_{\delta}(\bm{x}) \right\|_2
= \max_{i \in [d]} \left\{ 1 + 2\zeta e^{-\frac{2\delta^2}{d}}\cos(2x_i) \right\}
\geq 1 - 2\zeta e^{-\frac{2\delta^2}{d}}
\geq 1-2\zeta.
\end{align*}
For $\hat{f}_{\delta}$ to be strongly convex for any $\delta>0$, it is necessary that $1-2\zeta> 0$; i.e., $\zeta < \frac{1}{2}$ holds. Therefore, if $\zeta < 0.5$, then $f$ is a $(1-2\zeta)$-nice function.

\section{Proof of the Theorems and Propositions}
\subsection{Proof of Proposition \ref{prop:heavy}}\label{subsec:e.1}
\begin{proof}
From $|f(\bm{x})| = \mathcal{O}\left( \| \bm{x} \|^p \right)$ and Jensen's inequality,
\begin{align*}
|f(\bm{x} - \delta \bm{u})|
\leq C_1 \| \bm{x} - \delta \bm{u} \|^p
\leq C_1 \left( \|\bm{x}\| + \delta \| \bm{u} \| \right)^p
\leq C_1 \max\{1, 2^{p-1}\} \left( \| \bm{x} \|^p + \delta^p \| \bm{u} \|^p \right),
\end{align*}
where $C_1 > 0$ is a positive constant. Hence, we have
\begin{align*}
\mathbb{E}_{\bm{u}} \left[ |f(\bm{x} - \delta \bm{u})|\right]
\leq C_1 \max\{1, 2^{p-1}\} (\| \bm{x} \|^p + \delta^p \mathbb{E}_{\bm{u}}\left[ \| \bm{u} \|^p\right]).
\end{align*}
Therefore, it suffices to show that $\mathbb{E}_{\bm{u}}\left[ \| \bm{u} \|^p \right] < \infty$ holds.

\paragraph{(i) when $\bm{u}$ follows a light-tailed distribution.} Since $\bm{u}$ follows a light-tailed distribution, there exists $r > 0$, such that
\begin{align*}
\mathbb{E}_{\bm{u}}\left[ e^{r \| \bm{u} \|}\right] < \infty.
\end{align*} 
From the Taylor expansion, for all $p \geq 0$, we have
\begin{align*}
e^{r \| \bm{u} \|}
= \sum_{k=0}^{\infty} \frac{(r\| \bm{u} \|)^k}{k!}
\geq \frac{(r \| \bm{u} \|)^{\lceil p \rceil}}{\lceil p \rceil !}
\geq \frac{r^{\lceil p \rceil}}{\lceil p \rceil !}\left( \| \bm{u} \|^p -1 \right).
\end{align*}
Therefore, for all $p \geq 0$, we have
\begin{align*}
\mathbb{E}_{\bm{u}}\left[ \| \bm{u} \|^p \right]
\leq \frac{\lceil p \rceil !}{r^{\lceil p \rceil}} \left( \mathbb{E}_{\bm{u}}\left[ e^{r \| \bm{u} \|} \right] + 1 \right)
<\infty.
\end{align*}
This completes the proof for light-tailed distributions.

\paragraph{(ii) when $\bm{u}$ follows a heavy-tailed distribution with tail index $\alpha > p$.} By the layer cake representation,
\begin{align*}
\mathbb{E}_{\bm{u}}[\| \bm{u} \|^p] 
= p \int_{0}^{\infty} t^{p-1} P(\| \bm{u} \| > t) dt
= p \int_{0}^{T} t^{p-1} P(\| \bm{u} \| > t) dt + p \int_{T}^{\infty} t^{p-1} P(\| \bm{u} \| > t) dt 
\end{align*}
Since $\bm{u}$ follows a heavy-tailed distribution, we have $P(\| \bm{u} \| > t) \leq c_0 t^{-\alpha}$. Therefore, from $\alpha > p$,
\begin{align*}
\int_{T}^{\infty} t^{p-1} P(\| \bm{u} \| > t) dt 
\leq c_0 \int_{T}^{\infty} t^{p-\alpha-1} dt
= \frac{c_0}{\alpha-p}\cdot \frac{1}{T^{\alpha-p}}
< \infty.
\end{align*}
Hence, we have
\begin{align*}
\mathbb{E}_{\bm{u}}[\| \bm{u} \|^p] < \infty.
\end{align*}
This completes the proof for heavy-tailed distributions.
\end{proof}
}

\UPDATE{
\subsection{Lemmas for proof of Theorem \ref{thm:03}}
\begin{lem}\label{lem:xex}
Let $X \in \mathbb{R}^d$ be a random vector such that $\mathbb{E}[\|X\|^\alpha] < \infty$ for some $\alpha \geq 1$. For any constant vector $\bm{c} \in \mathbb{R}^d$ (or any random vector $\bm{c}$ independent of $X$), the following inequality holds:
\begin{align*}
\mathbb{E}[\|X - \mathbb{E}[X]\|^\alpha] \leq 2^\alpha \mathbb{E}[\|X - \bm{c}\|^\alpha].
\end{align*}
\end{lem}
\begin{proof}
By the triangle inequality and the linearity of expectation, we have
\begin{align*}
\| X - \mathbb{E}[X] \|
= \left\| \left( X-\bm{c} \right) - \left( \mathbb{E}[X] - \bm{c}\right) \right\|\
\leq \| X - \bm{x} \| + \| \mathbb{E}[X - \bm{c}] \|.
\end{align*}
Raising both sides to the power of $\alpha$ and applying the convexity of $\| \cdot \|^\alpha$, we have
\begin{align*}
\| X - \mathbb{E}[X] \|^\alpha
\leq 2^{\alpha-1} \left( \| X - \bm{c} \|^\alpha + \| \mathbb{E}[X - \bm{c}] \|^\alpha \right).
\end{align*}
Taking the expectation on both sides yields
\begin{align*}
\mathbb{E}[\|X - \mathbb{E}[X]\|^\alpha] 
\leq 2^{\alpha-1} \left( \mathbb{E}[\|X - \bm{c}\|^\alpha] + \mathbb{E}[\|\mathbb{E}[X - \bm{c}]\|^\alpha] \right)
\end{align*}
By Jensen's inequality, we have $\|\mathbb{E}[X - \bm{c}]\|^\alpha \leq \mathbb{E}[\|X - \bm{c}\|^\alpha]$. Therefore, 
\begin{align*}
\mathbb{E}[\|X - \mathbb{E}[X]\|^\alpha] 
\leq 2^{\alpha-1} \left( \mathbb{E}[\|X - \bm{c}\|^\alpha] + \mathbb{E}[\|X - \bm{c}\|^\alpha] \right) = 2^\alpha \mathbb{E}[\|X - \bm{c}\|^\alpha].
\end{align*}
This completes the proof.
\end{proof}

\begin{lem}
[Bound on the Virtual Noise Moment]\label{lem:nu} Let $\bm{y}_{t} := \bm{x}_t - \eta \nabla f(\bm{x}_t), \bm{x}_{t+1} := \bm{x}_t - \eta \nabla f_{\mathcal{S}_t}(\bm{x}_t)$, and $\bm{\omega}_t := \nabla f_{\mathcal{S}_t}(\bm{x}_t) - \nabla f(\bm{x}_t)$. Then, the virtual noise $\bm{\nu}_t := \bm{y}_{t+1} - \mathbb{E}_{\bm{\omega}_t}\left[ \bm{y}_{t+1}\right]$ satisfies the following bound:
\begin{align*}
\mathbb{E}\left[ \| \bm{\nu}_t \|^\alpha \right]
&\leq 2^{\alpha-1}\eta^\alpha (1+2^\alpha L_g^\alpha) \mathbb{E}\left[ \| \bm{\omega}_t \|^\alpha \right] \\
&\leq 2^{\alpha-1} (1+2^\alpha L_g^\alpha) \delta^\alpha.
\end{align*}
\end{lem}
\begin{proof}
From the definition of $\bm{y}_{t}$, we have
\begin{align*}
\bm{y}_{t+1} 
&= \bm{x}_{t+1} - \eta \nabla f(\bm{x}_{t+1})
= \bm{x}_t - \eta \nabla f_{\mathcal{S}_t}(\bm{x}_t) - \eta \nabla f(\bm{x}_{t+1})\\
&= \bm{x}_t - \eta (\bm{\omega}_t + \nabla f(\bm{x}_t)) - \eta \nabla f(\bm{x}_{t+1}).
\end{align*} 
Therefore, we have
\begin{align*}
\mathbb{E}_{\bm{\omega}_t}\left[ \bm{y}_{t+1}\right]
= \bm{x}_t - \eta \nabla f(\bm{x}_t) - \eta \mathbb{E}\left[ \nabla f(\bm{x}_{t+1})\right],
\end{align*}
and
\begin{align*}
\bm{\nu}_t = -\eta \bm{\omega}_t - \eta \left(\nabla f(\bm{x}_{t+1}) - \mathbb{E}_{\bm{\omega}_t}\left[ \nabla f(\bm{x}_{t+1})\right] \right).
\end{align*}
Hence, from the convexity of $\| \cdot \|^\alpha$,
\begin{align*}
\mathbb{E}\left[ \| \bm{\nu}_t \|^\alpha \right]
\leq 2^{\alpha-1} \eta^\alpha \left( \mathbb{E}\left[ \| \bm{\omega}_t \|^\alpha \right] + \mathbb{E}\left[ \| \nabla f(\bm{x}_{t+1}) - \mathbb{E}_{\bm{\omega}_t}\left[ \nabla f(\bm{x}_{t+1})\right] \|^\alpha \right] \right).
\end{align*}
For the second term, using Lemma \ref{lem:xex}, we have 
\begin{align*}
\mathbb{E}\left[ \| \nabla f(\bm{x}_{t+1}) - \mathbb{E}_{\bm{\omega}_t}\left[ \nabla f(\bm{x}_{t+1})\right] \|^\alpha \right] 
\leq 2^\alpha \mathbb{E}\left[ \| \nabla f(\bm{x}_{t+1}) - \nabla f(\bm{y}_t) \|^\alpha \right]
\end{align*}
where we choose $\bm{c} = \nabla f(\bm{y}_t)$. From the $L_g$-smoothness of $f$,
\begin{align*}
\mathbb{E}\left[ \| \nabla f(\bm{x}_{t+1}) - \mathbb{E}_{\bm{\omega}_t}\left[ \nabla f(\bm{x}_{t+1})\right] \|^\alpha \right] 
&\leq 2^\alpha L_g^\alpha \mathbb{E}\left[ \| \bm{x}_{t+1} - \bm{y}_t \|^\alpha \right] \\
&= 2^\alpha L_g^\alpha \mathbb{E}\left[ \| \bm{\omega}_t \|^\alpha \right].
\end{align*}
Hence, 
\begin{align*}
\mathbb{E}\left[ \| \bm{\nu}_t \|^\alpha \right]
\leq 2^{\alpha-1}\eta^\alpha (1+2^\alpha L_g^\alpha) \mathbb{E}\left[ \| \bm{\omega}_t \|^\alpha \right].
\end{align*}
In addition, from Lemma \ref{lem:00} and $\delta := \frac{\eta C_\alpha^{\frac{1}{\alpha}} \tau}{b^{\frac{\alpha-1}{\alpha}}}$, 
\begin{align*}
\mathbb{E}\left[ \| \bm{\nu}_t \|^\alpha \right]
\leq 2^{\alpha-1}\eta^\alpha (1+2^\alpha L_g^\alpha) \frac{C_\alpha \tau^\alpha}{b^{\alpha-1}}
= 2^{\alpha-1}(1+2^\alpha L_g^\alpha) \delta^\alpha.
\end{align*}
This completes the proof.
\end{proof}

\begin{lem}
[Generalized Descent Lemma, Lemma 1 in \citep{Yashtini2016Ont}] \label{lem:Holder} Assume that the function $f \colon \mathbb{R}^d \to \mathbb{R}$ satisfies the H\"{o}lder condition; i.e., there is a $\upsilon \in (0,1]$ and an $K>0$ such that, for all $\bm{x}, \bm{y} \in \mathbb{R}^d$, $\| \nabla f(\bm{x}) - \nabla f(\bm{y}) \| \leq K \| \bm{x} - \bm{y} \|^\upsilon$. Then, the following inequality holds.
\begin{align*}
f(\bm{y}) \leq f(\bm{x}) + \langle \nabla f(\bm{x}), \bm{y}-\bm{x} \rangle + \frac{K}{\upsilon + 1} \| \bm{y} - \bm{x} \|^{\upsilon+1}.
\end{align*}
\end{lem}

\subsection{Proof of Theorem \ref{thm:03}}
\label{subsec:9.3}
\begin{proof}
From the $L_g$-smoothness of $\hat{f}_{\delta_m}$, for all $\bm{x}, \bm{y} \in \mathbb{R}^d$, we have
\begin{align*}
\| \nabla \hat{f}_{\delta_m}(\bm{x}) - \nabla \hat{f}_{\delta_m}(\bm{y}) \| 
\leq L_g \| \bm{x} - \bm{y} \|
\Longleftrightarrow
\| \nabla \hat{f}_{\delta_m}(\bm{x}) - \nabla \hat{f}_{\delta_m}(\bm{y}) \|^{\alpha-1} 
\leq L_g^{\alpha-1} \| \bm{x} - \bm{y} \|^{\alpha-1}.
\end{align*}
In addition, from the $L_f$-Lipschitz continuity of $\hat{f}_{\delta_m}$, for all $\bm{x}, \bm{y} \in \mathbb{R}^d$, we have
\begin{align*}
&\| \nabla \hat{f}_{\delta_m}(\bm{x}) - \nabla \hat{f}_{\delta_m}(\bm{y}) \| 
\leq \| \nabla \hat{f}_{\delta_m}(\bm{x}) \| + \| \nabla \hat{f}_{\delta_m}(\bm{y}) \|
\leq 2L_f \\
&\quad
\Longleftrightarrow
\| \nabla \hat{f}_{\delta_m}(\bm{x}) - \nabla \hat{f}_{\delta_m}(\bm{y}) \|^{2-\alpha} \leq (2L_f)^{2-\alpha}.
\end{align*}
Multiplying these inequalities, for all $\bm{x}, \bm{y} \in \mathbb{R}^d$, gives
\begin{align*}
\| \nabla \hat{f}_{\delta_m}(\bm{x}) - \nabla \hat{f}_{\delta_m}(\bm{y}) \| 
\leq \underbrace{L_g^{\alpha-1}(2L_f)^{2-\alpha}}_{=:K}\| \bm{x}-\bm{y}\|^{\alpha-1}.
\end{align*}
Therefore, from Lemma \ref{lem:Holder} and $\hat{\bm{x}}_{t+1}^{(m)} := \hat{\bm{x}}_t^{(m)} - \eta_m \nabla \hat{f}_{\delta_m}(\hat{\bm{x}}_t^{(m)}) + \bm{\nu}_t^{(m)}$, we have
\begin{align*}
\hat{f}_{\delta_m}(\hat{\bm{x}}_{t+1}^{(m)})
&\leq \hat{f}_{\delta_m}(\hat{\bm{x}}_{t}^{(m)}) + \langle \nabla \hat{f}_{\delta_m}(\hat{\bm{x}}_t^{(m)}), \hat{\bm{x}}_{t+1}^{(m)} - \hat{\bm{x}}_t^{(m)} \rangle + \frac{K}{\alpha}\| \hat{\bm{x}}_{t+1}^{(m)} - \hat{\bm{x}}_t^{(m)} \|^\alpha \\
&= \hat{f}_{\delta_m}(\hat{\bm{x}}_{t}^{(m)}) - \langle \nabla \hat{f}_{\delta_m}(\hat{\bm{x}}_t^{(m)}), \eta_m \nabla \hat{f}_{\delta_m}(\hat{\bm{x}}_t^{(m)}) - \bm{\nu}_t^{(m)} \rangle + \frac{K}{\alpha}\| \eta_m \nabla \hat{f}_{\delta_m}(\hat{\bm{x}}_t^{(m)}) - \bm{\nu}_t^{(m)} \|^\alpha \\
&\leq \hat{f}_{\delta_m}(\hat{\bm{x}}_{t}^{(m)}) - \eta_m \| \nabla \hat{f}_{\delta_m}(\hat{\bm{x}}_t^{(m)}) \|^2 + \langle \nabla \hat{f}_{\delta}(\hat{\bm{x}}_t^{(m)}), \bm{\nu}_t^{(m)} \rangle + \frac{K}{\alpha}\| \eta_m \nabla \hat{f}_{\delta_m}(\hat{\bm{x}}_t^{(m)}) - \bm{\nu}_t^{(m)} \|^\alpha.
\end{align*}
By taking expectations, 
\begin{align*}
\mathbb{E}\left[ \hat{f}_{\delta_m}(\hat{\bm{x}}_{t+1}^{(m)}) \right]
\leq \mathbb{E}\left[ \hat{f}_{\delta_m}(\hat{\bm{x}}_{t}^{(m)}) \right] - \eta_m \mathbb{E}\left[ \| \nabla \hat{f}_{\delta_m}(\hat{\bm{x}}_t^{(m)}) \|^2\right] + \frac{K}{\alpha} \mathbb{E}\left[ \| \eta_m \nabla \hat{f}_{\delta_m}(\hat{\bm{x}}_t^{(m)}) - \bm{\nu}_t^{(m)} \|^\alpha \right].
\end{align*}
Here, from $\| \bm{x} - \bm{y} \|^\alpha \leq \| \bm{x} \|^\alpha - \alpha \| \bm{x} \|^{\alpha-2} \langle \bm{x}, \bm{y} \rangle + \| \bm{y} \|^\alpha $, we have
\begin{align*}
\| \eta_m \nabla \hat{f}_{\delta_m}(\hat{\bm{x}}_t^{(m)}) - \bm{\nu}_t^{(m)} \|^\alpha
\leq \eta_m^\alpha \| \nabla \hat{f}_{\delta_m}(\hat{\bm{x}}_t^{(m)}) \|^\alpha - \alpha \eta_m^\alpha \| \nabla \hat{f}_{\delta_m}(\hat{\bm{x}}_t^{(m)}) \|^{\alpha-2} \langle \nabla \hat{f}_{\delta_m}(\hat{\bm{x}}_t^{(m)}), \bm{\nu}_t^{(m)}\rangle + \| \bm{\nu}_t^{(m)} \|^\alpha.
\end{align*}
In addition, from Young's inequality $ab \leq \frac{a^p}{p} + \frac{b^q}{q}$ $\left(\frac{1}{p} + \frac{1}{q} =1, p\geq1, q\geq 1\right)$, 
\begin{align*}
\| \nabla \hat{f}_{\delta_m}(\bm{x}_t^{(m)}) \|^{\alpha} \cdot 1
\leq \frac{\alpha}{2}\| \nabla \hat{f}_{\delta_m}(\bm{x}_t^{(m)}) \|^2 + \frac{2- \alpha}{2},
\end{align*}
where we choose $a = \| \nabla \hat{f}_{\delta_m}(\bm{x}_t^{(m)}) \|^\alpha, b=1, p=\frac{2}{\alpha}$, and $q= \frac{2}{2-\alpha}$. Therefore, from Lemma \ref{lem:nu},
\begin{align*}
\mathbb{E}\left[ \| \eta_m \nabla \hat{f}_{\delta_m}(\hat{\bm{x}}_t^{(m)}) - \bm{\nu}_t^{(m)} \|^\alpha \right]
&\leq \eta_m^\alpha \mathbb{E}\left[ \| \nabla \hat{f}_{\delta_m}(\hat{\bm{x}}_t^{(m)}) \|^\alpha \right] + \mathbb{E}\left[ \| \bm{\nu}_t^{(m)} \|^\alpha \right] \\
&\leq \frac{\alpha \eta_m^\alpha }{2} \mathbb{E}\left[ \| \nabla \hat{f}_{\delta_m}(\bm{x}_t^{(m)}) \|^2 \right] + \frac{(2-\alpha)\eta_m^\alpha}{2} + 2^{\alpha-1}(1+2^\alpha L_g^\alpha) \delta_m^\alpha
\end{align*}
Hence, we have
\begin{align*}
\mathbb{E}\left[ \hat{f}_{\delta_m}(\hat{\bm{x}}_{t+1}^{(m)}) \right]
&\leq \mathbb{E}\left[ \hat{f}_{\delta_m}(\hat{\bm{x}}_{t}^{(m)}) \right] - \eta_m\left( 1- \frac{K\eta_m^{\alpha-1}}{2} \right) \mathbb{E}\left[ \| \nabla \hat{f}_{\delta_m}(\bm{x}_t^{(m)}) \|^2 \right] \\
&\quad + \frac{K}{\alpha} \left\{ \frac{(2-\alpha)\eta_m^\alpha}{2} + 2^{\alpha-1}(1+2^\alpha L_g^\alpha) \delta_m^\alpha \right\}
\end{align*}
On the other hand, from $\mu_1$-PL condition, we have
\begin{align*}
\frac{1}{2}\| \nabla \hat{f}_{\delta_m}(\hat{\bm{x}}_t^{(m)}) \|^2
\geq \mu_1 \left(\hat{f}_{\delta_m}(\hat{\bm{x}}_t^{(m)}) - \hat{f}_{\delta_m}(\bm{x}_{\delta_m}^\star) \right).
\end{align*}
Therefore, we have
\begin{align*}
\mathbb{E}\left[ \hat{f}_{\delta_m}(\hat{\bm{x}}_{t+1}^{(m)}) \right]
&\leq \mathbb{E}\left[ \hat{f}_{\delta_m}(\hat{\bm{x}}_t^{(m)}) \right] - 2\mu_1 \eta_m\left( 1- \frac{K\eta_m^{\alpha-1}}{2} \right) \left( \mathbb{E}\left[ \hat{f}_{\delta_m}(\hat{\bm{x}}_t^{(m)}) \right] - \hat{f}_{\delta_m}(\hat{\bm{x}}_{\delta_m}^\star)\right) \\
&\quad+ \underbrace{\frac{K}{\alpha} \left\{ \frac{(2-\alpha)\eta_m^\alpha}{2} + 2^{\alpha-1}(1+2^\alpha L_g^\alpha) \delta_m^\alpha \right\}}_{=: \rho_1}.
\end{align*}
Subtracting $\hat{f}_{\delta_m}(\hat{\bm{x}}_{\delta_m}^\star)$ from both sides yields
\begin{align*}
&\mathbb{E}\left[ \hat{f}_{\delta_m}(\hat{\bm{x}}_{t+1}^{(m)}) \right] - \hat{f}_{\delta_m}(\bm{x}_{\delta_m}^\star) \\
&\quad\leq \mathbb{E}\left[ \hat{f}_{\delta_m}(\hat{\bm{x}}_t^{(m)}) \right] - \hat{f}_{\delta_m}(\bm{x}_{\delta_m}^\star)
-2\mu_1 \eta_m\left( 1- \frac{K\eta_m^{\alpha-1}}{2} \right) \left( \mathbb{E}\left[ \hat{f}_{\delta_m}(\hat{\bm{x}}_t^{(m)}) \right] - \hat{f}_{\delta_m}(\hat{\bm{x}}_{\delta_m}^\star)\right)
+ \rho_1 \\
&\quad= \underbrace{\left( 1 - 2\mu_1 \eta_m\left( 1- \frac{K\eta_m^{\alpha-1}}{2}\right) \right)}_{=: \rho_2} \left( \mathbb{E}\left[ \hat{f}_{\delta_m}(\hat{\bm{x}}_t^{(m)}) \right] - \hat{f}_{\delta_m}(\bm{x}_{\delta_m}^\star)\right) + \rho_1 \\
&\quad\leq \rho_2^{t} \left( \hat{f}_{\delta_m}(\hat{\bm{x}}_1^{(m)}) - \hat{f}_{\delta_m}(\bm{x}_{\delta_m}^\star)\right) + \frac{\rho_1}{1-\rho_2},
\end{align*}
where $|\rho_2| < 1$ since $\eta_m < \min\left\{ \frac{1}{\mu_1}, \left( \frac{2}{K} \right)^{\frac{1}{\alpha-1}}\right\}$. Summing over $t$, we find that
\begin{align*}
\sum_{t=1}^{T} \left( \mathbb{E}\left[ \hat{f}_{\delta_m}(\hat{\bm{x}}_{t+1}^{(m)}) \right] - \hat{f}_{\delta_m}(\bm{x}_{\delta_m}^\star) \right)
&\leq \left( \hat{f}_{\delta_m}(\hat{\bm{x}}_1^{(m)}) - \hat{f}_{\delta_m}(\bm{x}_{\delta_m}^\star)\right) \sum_{t=1}^{T} \rho_2^t + \frac{\rho_1 T}{(1-\rho_2)}\\
&\leq \frac{ \hat{f}_{\delta_m}(\hat{\bm{x}}_1^{(m)}) - \hat{f}_{\delta_m}(\bm{x}_{\delta_m}^\star)}{1-\rho_2} + \frac{\rho_1 T}{(1-\rho_2)}.
\end{align*}
Here, from $L_f$-Lipschitzness of $\hat{f}_{\delta_m}$ and the unique $\mu_2$-QG condition, we have 
\begin{align*}
\hat{f}_{\delta_m}(\hat{\bm{x}}_1^{(m)}) - \hat{f}_{\delta_m}(\bm{x}_{\delta_m}^\star) 
\leq \frac{2L_f^2}{\mu_2}.
\end{align*}
Hence,
\begin{align*}
&\frac{1}{T} \sum_{t=1}^{T} \left( \mathbb{E}\left[ \hat{f}_{\delta_m}(\hat{\bm{x}}_{t+1}^{(m)}) \right] - \hat{f}_{\delta_m}(\bm{x}_{\delta_m}^\star) \right) \\
&\quad\leq \frac{2L_f^2}{\mu_2 (1-\rho_2)} \cdot \frac{1}{T} + \frac{\rho_1}{(1-\rho_2)} \\
&\quad= \underbrace{\frac{2L_f^2}{\mu_1\mu_2\left( 2-K\eta_m^{\alpha-1}\right)}}_{=: H_m} \cdot \frac{1}{T} + \underbrace{\frac{K\left\{ (2-\alpha)\eta_m^\alpha + 2^{\alpha}(1+2^\alpha L_g^\alpha) \delta_m^\alpha \right\}}{\alpha\mu_1\mu_2\left( 2-K\eta_m^{\alpha-1}\right)}}_{=:I_m},
\end{align*}
where $H_m > 0$ and $I_m > 0$ are nonnegative constants. Since the minimum value is smaller than the mean, we have
\begin{align*}
\min_{t \in [T]} \left( \mathbb{E}\left[ \hat{f}_{\delta_m} \left(\hat{\bm{x}}_t^{(m)} \right) \right] - \hat{f}_{\delta_m}(\bm{x}_{\delta_m}^\star) \right)
\leq \frac{H_m}{T} + I_m
= \mathcal{O} \left( \frac{1}{T} + \delta_m^\alpha \right).
\end{align*}
This completes the proof.
\end{proof}
}

\subsection{Proof of Proposition \ref{prop:999}}\label{sec:E.3}
\begin{proof}
This proposition can be proved by induction. Since we assume $\bm{x}_1 \in B(\bm{x}_{\delta_1}^\star; 3\delta_1)$, we have
\begin{align*}
\left\| \bm{x}_1 - \bm{x}_{\delta_1}^\star \right\| < 3\delta_1,
\end{align*}
which establishes the case of $m=1$. Now let us assume that the proposition holds for any $m>1$. Accordingly, the initial point $\bm{x}_m$ for the optimization of the $m$-th smoothed function $\hat{f}_{\delta_m}$ and its global optimal solution $\bm{x}_{\delta_m}^\star$ are both contained in the local favorable region $B(\bm{x}_{\delta_m}^\star; 3\delta_m)$. Thus, after $T_m := H_m/(\epsilon_m - I_m)$ iterations, Algorithm \ref{alg:sgd2} (GD) returns an approximate solution $\hat{\bm{x}}_{T_m +1}^{(m)} =: \bm{x}_{m+1}$, and the following holds from Theorem \ref{thm:03}:
\begin{align*}
\mathbb{E}\left[ \hat{f}_{\delta_m}(\bm{x}_{m+1}) \right] - \hat{f}_{\delta_m}(\bm{x}_{\delta_m}^\star) \leq \frac{H_m}{T_m} + I_m = \epsilon_m := \frac{\mu_2 \delta_m^2}{2} = \frac{\mu_2 \delta_{m+1}^2}{2\gamma^2}.
\end{align*}
Hence, from the unique $\mu_2$-QG of $\hat{f}_{\delta_m}$, 
\begin{align*}
\frac{\mu_2}{2} \mathbb{E}\left[ \left\| \bm{x}_{m+1} - \bm{x}_{\delta_m}^\star \right\|^2 \right] 
\leq \frac{\mu_2 \delta_{m+1}^2}{2\gamma^2}, \ \text{i.e., }
\mathbb{E}\left[ \left\| \bm{x}_{m+1} - \bm{x}_{\delta_m}^\star \right\| \right]
\leq \frac{\delta_{m+1}}{\gamma}
\end{align*}
Therefore, from the weakly $(\mu_1, \mu_2)$-nice property (i) and $\gamma \in [0.5,1)$, 
\begin{align*}
\mathbb{E}\left[ \left\| \bm{x}_{m+1} - \bm{x}_{\delta_{m+1}}^\star \right\| \right]
&\leq \mathbb{E}\left[ \left\| \bm{x}_{m+1} - \bm{x}_{\delta_m}^\star \right\| \right] + \mathbb{E}\left[ \left\| \bm{x}_{\delta_m}^\star - \bm{x}_{\delta_{m+1}}^\star \right\| \right] \\
&\leq \frac{\delta_{m+1}}{\gamma} + (1-\gamma) \delta_m \\
&= \frac{\delta_{m+1}}{\gamma} + (1-\gamma) \frac{\delta_{m+1}}{\gamma} \\ 
&= \left( \frac{2}{\gamma} -1 \right)\delta_{m+1} \\
&\leq 3\delta_{m+1}.
\end{align*}
This completes the proof.
\end{proof}

\subsection{Proof of Theorem \ref{thm:3.4}}
\label{subsec:9.4}
The following proof uses the technique presented in \citep{Elad2016OnG}.
\begin{proof}
According to $\delta_{m+1}:= \frac{\eta_{m+1}C_{\alpha}^{1/\alpha}\tau}{b_{m+1}^{(\alpha-1)/\alpha}}$ and $\frac{\kappa_m}{\lambda_m^{(\alpha-1)/\alpha}} = \gamma$, we have
\begin{align*}
\delta_{m+1} 
:= \frac{\eta_{m+1}C_{\alpha}^{1/\alpha}\tau}{b_{m+1}^{(\alpha-1)/\alpha}}
=\frac{\kappa_m \eta_m C}{(\lambda_m b_m)^{(\alpha-1)/\alpha}}
=\frac{\kappa_m}{\lambda_m^{(\alpha-1)/\alpha}}\delta_m
=\gamma \delta_m.
\end{align*}
Therefore, from $M := \log_{\gamma}(\beta_0\epsilon)$,
\begin{align*}
\delta_M
= \delta_1 \gamma^{M-1}
=\frac{\delta_1 \beta_0 \epsilon}{\gamma}.
\end{align*}
According to Theorem \ref{thm:03}, 
\begin{align*}
\mathbb{E}\left[ \hat{f}_{\delta_M}(\bm{x}_{M+1}) - \hat{f}_{\delta_M}(\bm{x}_{\delta_M}^\star) \right]
&\leq \epsilon_M := \frac{\mu_2\delta_M^2}{2} = \frac{\mu_2 \delta_1^2 \beta_0^2 \epsilon^2}{2\gamma^2}.
\end{align*}
From Lemmas \ref{lem:04} and \ref{lem:06}, 
\begin{align*}
\mathbb{E}\left[ f(\bm{x}_{M+1}) - f(\bm{x}^\star) \right]
&=\mathbb{E}\left[ \left\{f(\bm{x}_{M+1}) - \hat{f}_{\delta_M}(\bm{x}_{M+1})\right\} + \left\{\hat{f}_{\delta_M}(\bm{x}^\star) - f(\bm{x}^\star)\right\} + \left\{\hat{f}_{\delta_M}(\bm{x}_{M+1}) - \hat{f}_{\delta_M}(\bm{x}^\star) \right\}\right] \\
&\leq \mathbb{E}\left[ \left\{f(\bm{x}_{M+1}) - \hat{f}_{\delta_M}(\bm{x}_{M+1})\right\} + \left\{\hat{f}_{\delta_M}(\bm{x}^\star) - f(\bm{x}^\star)\right\} + \left\{\hat{f}_{\delta_M}(\bm{x}_{M+1}) - \hat{f}_{\delta_M}(\bm{x}_{\delta_M}^\star)\right\}\right] \\
&\leq 2\delta_M L_f + \epsilon_M \\
&\leq \left( \frac{2L_f\delta_1 \beta_0}{\gamma} + \frac{\mu_2\delta_1^2\beta_0^2}{2\gamma^2} \right)\epsilon \\
&\leq \epsilon,
\end{align*}
where we have used $\beta_0 \leq \min\left\{ \frac{\gamma}{4L_f\delta_1}, \frac{\gamma}{\sqrt{\mu_2}\delta_1} \right\}$.\\
Let $T_{\text{total}}$ be the total number of queries made by Algorithm \ref{alg:gnc2}; then,
\begin{align*}
T_{\text{total}} 
&= \sum_{m=1}^{M} \frac{H_m}{\epsilon_m - I_m}
\geq \sum_{m=1}^{M} \frac{H_m}{\epsilon_m}
= \sum_{m=1}^{M} \frac{2L_f^2}{\mu_1\mu_2^2\eta_m\delta_m^2}
\geq \frac{2L_f^2}{\mu_1\mu_2^2\eta_1}\sum_{m=1}^{M} \frac{1}{\delta_m^2}.
\end{align*}
Here, from $\delta_{m+1} := \gamma \delta_m$ and $M := \log_{\gamma} \beta_0 \epsilon$,
\begin{align*}
\sum_{m=1}^{M}\frac{1}{\delta_m^2}
= \sum_{m=1}^{M} \frac{1}{\delta_1^2 \gamma^{2(m-1)}}
= \frac{1}{\delta_1^2}\sum_{m=0}^{M-1} \frac{1}{\gamma^{2m}}
= \frac{1}{\delta_1^2} \cdot \frac{\left(\frac{1}{\gamma^2}\right)^{M-1} - \gamma^2}{1-\gamma^2}
= \frac{\gamma^2}{\delta_1^2} \cdot \frac{\left( \frac{1}{\beta_0 \epsilon}\right)^2 - 1}{1-\gamma^2}.
\end{align*}
Therefore, we have
\begin{align*}
T_{\text{total}} 
\geq \frac{2L_f^2 \gamma^2}{\mu_1\mu_2^2\eta_1\delta_1^2(1-\gamma^2)} \left\{ \left( \frac{1}{\beta_0 \epsilon}\right)^2 - 1\right\}
= \Omega\left( \frac{1}{\epsilon^2}\right).
\end{align*}
This completes the proof.
\end{proof}

\section{Numerical Results}
\label{sec:exp}
This section presents the results of numerical experiments to support the theory.
The experimental environment was as follows: NVIDIA GeForce RTX 4090$\times$2GPU and Intel Core i9 13900KF CPU. The software environment was Python 3.10.12, PyTorch 2.1.0 and CUDA 12.2. 

\subsection{Numerical experiments on implicit graduated optimization}
We compared four types of SGD for image classification: {\em 1.} constant learning rate and constant batch size, {\em 2.} decaying learning rate and constant batch size, {\em 3.} constant learning rate and increasing batch size, {\em 4.} decaying learning rate and increasing batch size, in training ResNet34 \citep{He2016Dee} on the ImageNet dataset \citep{Deng2009Ima} (Figure \ref{fig:ImageNetApp}), ResNet18 \citep{He2016Dee} on the CIFAR100 dataset \citep{Alex2009Lea} (Figure \ref{fig:01}), and WideResNet-28-10 \citep{Zagoruyko2016Wid} on the CIFAR100 dataset (Figure \ref{fig:02}). Therefore, methods 2, 3, and 4 are our Algorithm \ref{alg:gnc2}. All experiments were run for 200 epochs. In methods 2, 3, and 4, the noise decreased every 40 epochs, with a common decay rate of $1/\sqrt{2}$. That is, every 40 epochs, the learning rate of method 2 was multiplied by $1/\sqrt{2}$, the batch size of method 3 was doubled, and the learning rate and batch size of method 4 were respectively multiplied by $\sqrt{3}/2$ and $1.5$. Note that this $1/\sqrt{2}$ decay rate is $\gamma$ in Algorithm \ref{alg:gnc2} and it satisfies the condition in Proposition \ref{prop:999}. The initial learning rate was 0.1 for all methods, which was determined by performing a grid search among $\left[0.01, 0.1, 1.0, 10\right]$. The noise reduction interval was every 40 epochs, which was determined by performing a grid search among $\left[10, 20, 25, 40, 50, 100\right]$. A history of the learning rate or batch size for each method is provided in the caption of each figure.

\begin{figure}[h]
\begin{tabular}{cc}
\begin{minipage}[t]{0.49\hsize}
\centering
\includegraphics[width=1\textwidth]{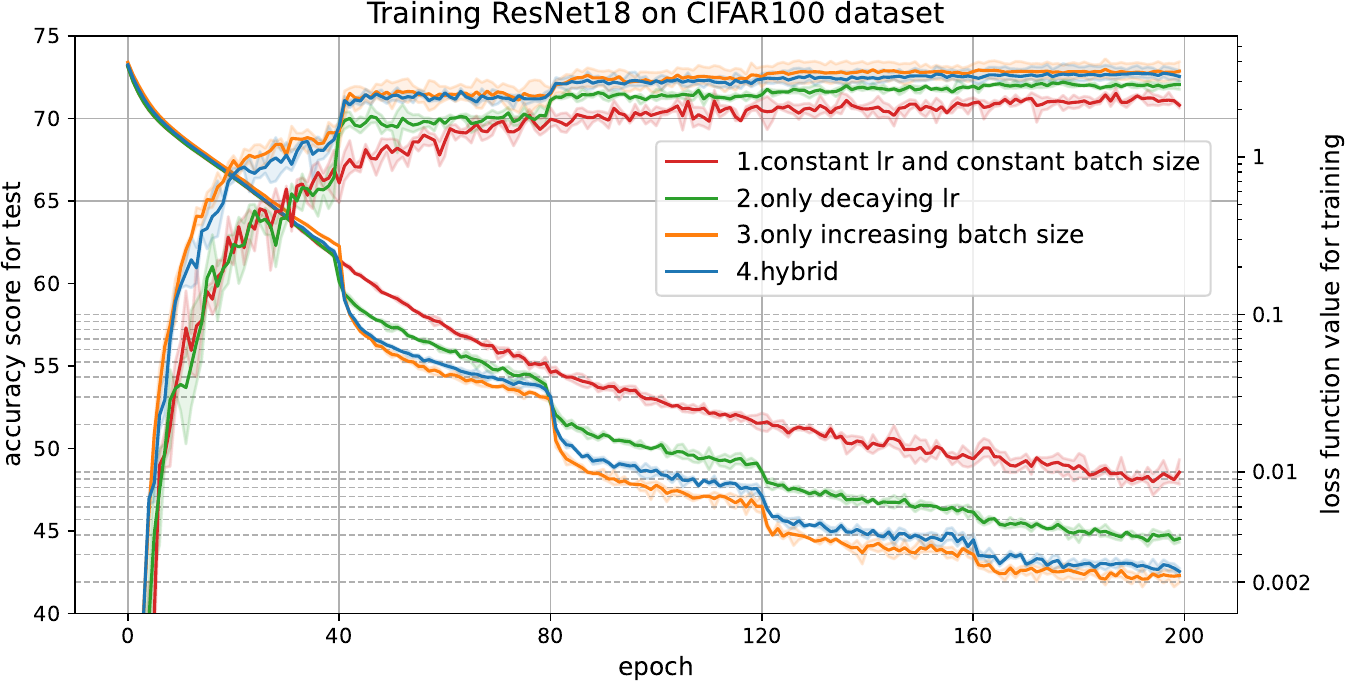}
\end{minipage} 
\begin{minipage}[t]{0.49\hsize}
\centering
\includegraphics[width=1\textwidth]{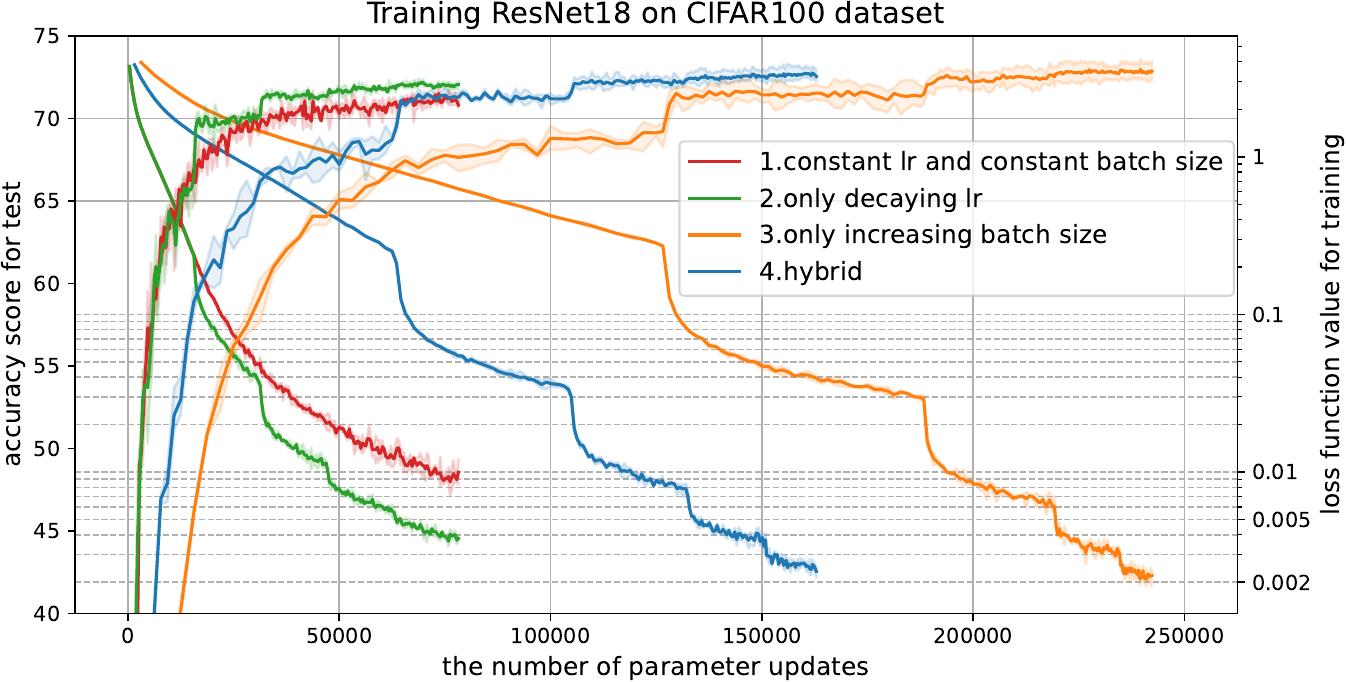}
\end{minipage}
\end{tabular}
\caption{Accuracy score for testing and loss function value for training versus the number of epochs (\textbf{left}) and the number of parameter updates (\textbf{right}) in training ResNet18 on the CIFAR100 dataset. The solid line represents the mean value, and the shaded area represents the maximum and minimum over three runs. In method 1, the learning rate and the batch size were fixed at 0.1 and 128, respectively. In method 2, the learning rate decreased every 40 epochs as $\left[0.1, \frac{1}{10\sqrt{2}}, 0.05, \frac{1}{20\sqrt{2}}, 0.025\right]$ and the batch size was fixed at 128. In method 3, the learning rate was fixed at 0.1, and the batch size was increased as $\left[16, 32, 64, 128, 256\right]$. In method 4, the learning rate was decreased as $\left[0.1, \frac{\sqrt{3}}{20}, 0.075, \frac{3\sqrt{3}}{80}, 0.05625\right]$ and the batch size was increased as $\left[32, 48, 72, 108, 162\right]$.}
\label{fig:01}
\end{figure}

\begin{figure}[h]
\begin{tabular}{cc}
\begin{minipage}[t]{0.49\hsize}
\centering
\includegraphics[width=1\textwidth]{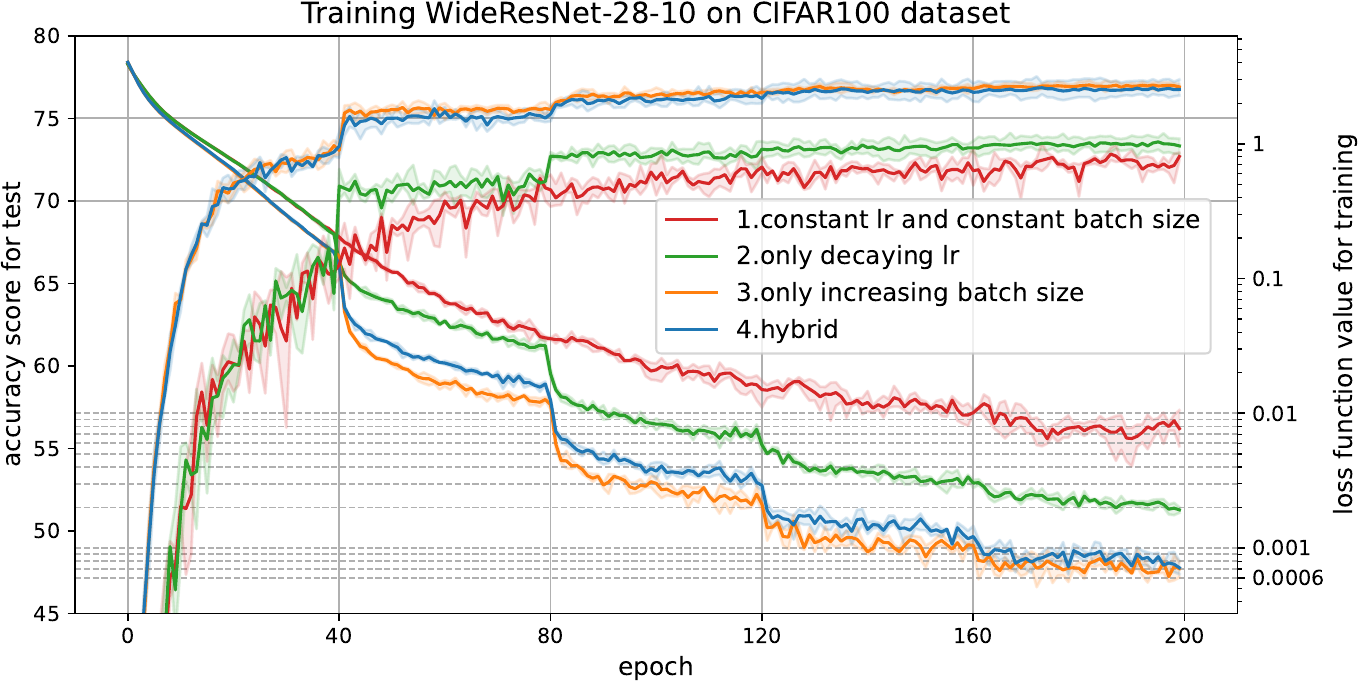}
\end{minipage} 
\begin{minipage}[t]{0.49\hsize}
\centering
\includegraphics[width=1\textwidth]{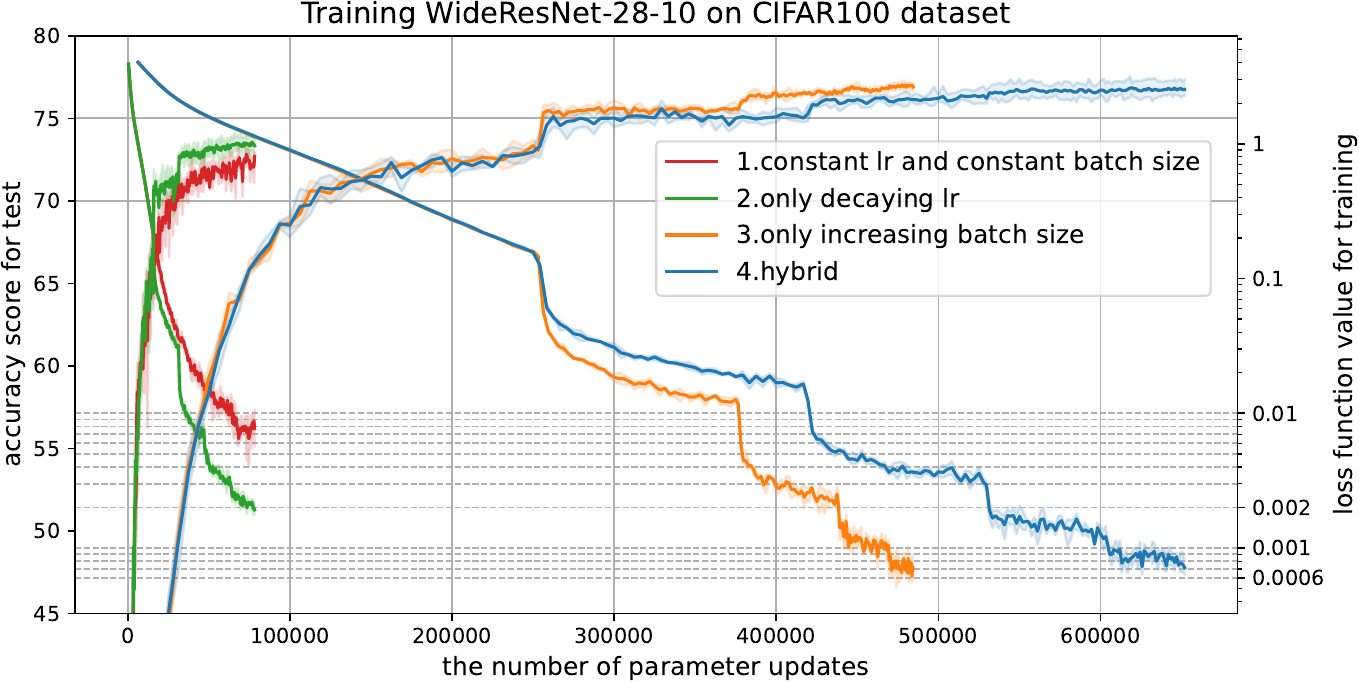}
\end{minipage}
\end{tabular}
\caption{Accuracy score for testing and loss function value for training versus the number of epochs (\textbf{left}) and the number of parameter updates (\textbf{right}) in training WideResNet-28-10 on the CIFAR100 dataset. The solid line represents the mean value, and the shaded area represents the maximum and minimum over three runs. In method 1, the learning rate and batch size were fixed at 0.1 and 128, respectively. In method 2, the learning rate was decreased every 40 epochs as $\left[0.1, \frac{1}{10\sqrt{2}}, 0.05, \frac{1}{20\sqrt{2}}, 0.025\right]$ and the batch size was fixed at 128. In method 3, the learning rate was fixed at 0.1, and the batch size was increased as $\left[8, 16, 32, 64, 128\right]$. In method 4, the learning rate was decreased as $\left[0.1, \frac{\sqrt{3}}{20}, 0.075, \frac{3\sqrt{3}}{80}, 0.05625\right]$ and the batch size was increased as $\left[8, 12, 18, 27, 40\right]$.}
\label{fig:02}
\end{figure}

\begin{figure}[h]
\begin{tabular}{cc}
\begin{minipage}[t]{0.49\hsize}
\centering
\includegraphics[width=1\linewidth]{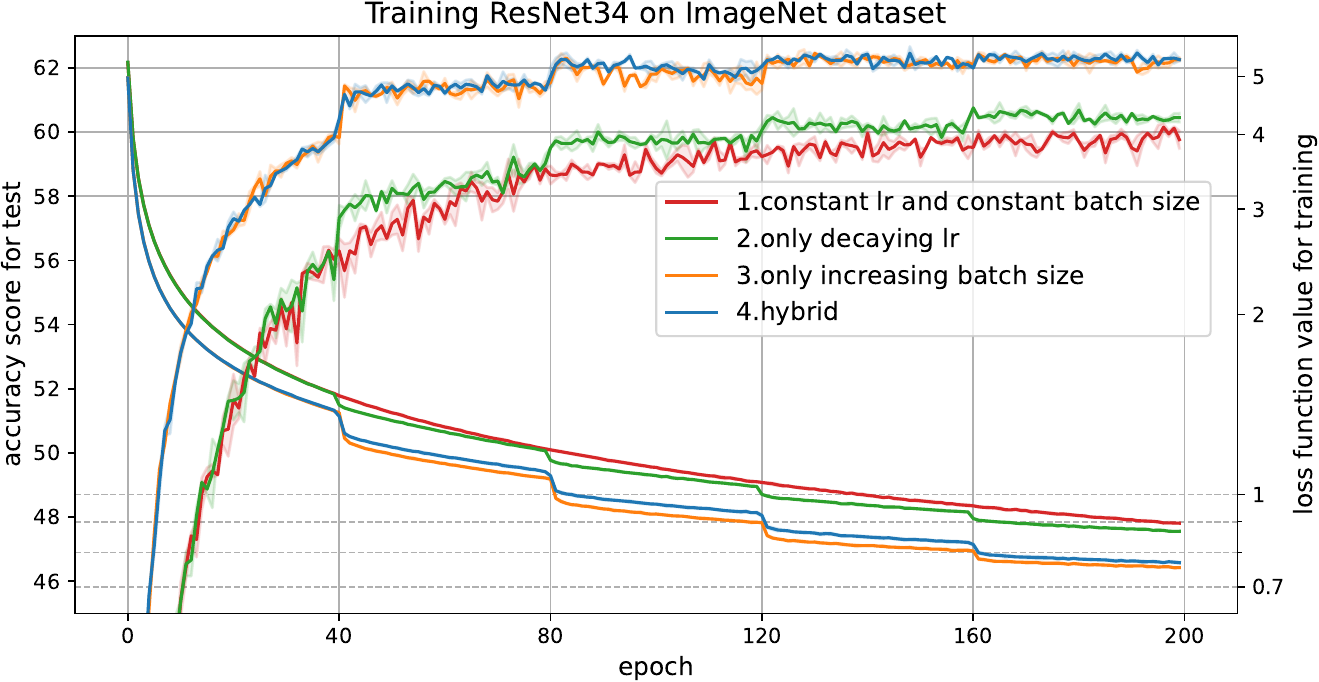}%
\end{minipage} 
\begin{minipage}[t]{0.49\hsize}
\centering
\includegraphics[width=1\textwidth]{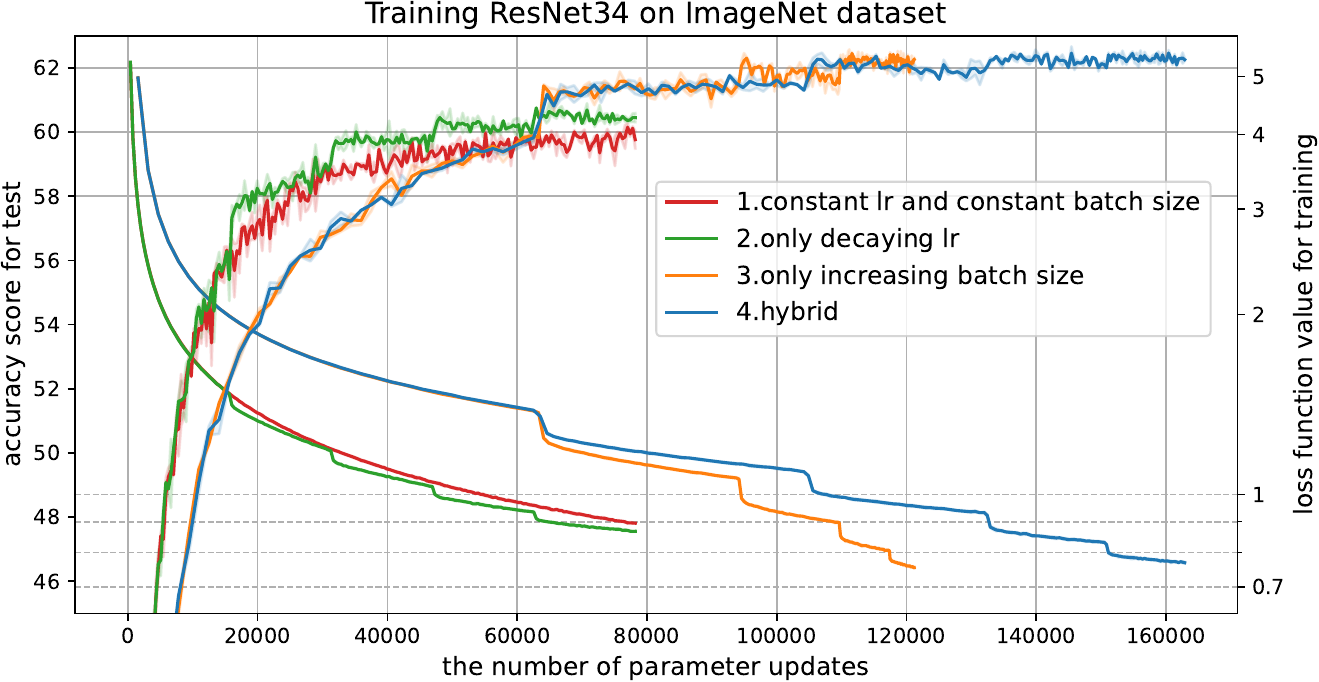}%
\end{minipage}
\end{tabular}
\caption{Accuracy score for the testing and loss function value for training versus the number of epochs (\textbf{left}) and the number of parameter updates (\textbf{right}) in training ResNet34 on the ImageNet dataset. The solid line represents the mean value, and the shaded area represents the maximum and minimum over three runs. In method 1, the learning rate and batch size were fixed at 0.1 and 256, respectively. In method 2, the learning rate was decreased every 40 epochs as $\left[0.1, \frac{1}{10\sqrt{2}}, 0.05, \frac{1}{20\sqrt{2}}, 0.025\right]$ and the batch size was fixed at 256. In method 3, the learning rate was fixed at 0.1, and the batch size was increased as $\left[32, 64, 128, 256, 512\right]$. In method 4, the learning rate was decreased as $\left[0.1, \frac{\sqrt{3}}{20}, 0.075, \frac{3\sqrt{3}}{80}, 0.05625\right]$ and the batch size was increased as $\left[32, 48, 72, 108, 162\right]$.}
\label{fig:ImageNetApp}
\end{figure}

For methods 2, 3, and 4, the decay rates are all $1/\sqrt{2}$, and the decay intervals are all 40 epochs, so throughout the training, the three methods should theoretically be optimizing the exact same five smoothed functions in sequence. Nevertheless, the local solutions reached by each of the three methods are not exactly the same. All results indicate that method 3 is superior to method 2 and that method 4 is superior to method 3 in both test accuracy and training loss function values. This difference can be attributed to the different learning rates used to optimize each smoothing function. Among methods 2, 3, and 4, method 3, which does not decay the learning rate, maintains the highest learning rate 0.1, followed by method 4 and method 2. In all graphs, the loss function values are always small in that order; i.e., the larger the learning rate is, the lower loss function values become. Therefore, we can say that the noise level $\delta$, expressed as $\frac{\eta C}{\sqrt{b}}$, needs to be reduced, while the learning rate $\eta$ needs to remain as large as possible. Alternatively, if the learning rate is small, then a large number of iterations are required. Thus, for the same rate of change and the same number of epochs, an increasing batch size is superior to a decreasing learning rate because it can maintain a large learning rate and can be made to iterate a lot when the batch size is small.

Theoretically, the noise level $\delta_m$ should gradually decrease and become zero at the end, so in our algorithm \ref{alg:gnc2}, the learning rate $\eta_m$ should be zero at the end or the batch size $b_m$ should match the number of data sets at the end. However, if the learning rate is 0, training cannot proceed, and if the batch size is close to a full batch, it is not feasible from a computational point of view. For this reason, the experiments described in this paper are not fully graduated optimizations; i.e., full global optimization is not achieved. In fact, the last batch size used by method 2 is around 128 to 512, which is far from a full batch. Therefore, the solution reached in this experiment is the optimal one for a function that has been smoothed to some extent, and to achieve a global optimization of the DNN, it is necessary to increase only the batch size to eventually reach a full batch, or increase the number of iterations accordingly while increasing the batch size and decaying the learning rate.

\subsection{Relationship between degree of smoothing, sharpness, and generalizability}
\label{sec:331}
The graduated optimization algorithm is a method in which the degree of smoothing $\delta$ is gradually decreased. Let us consider the case where the degree of smoothing $\delta$ is constant throughout the training. The following lemma shows the relationship between the error of the original function value and that of the smoothed function value.

\begin{lem}\label{lem:06}
Let $\hat{f}_\delta$ be the smoothed version of $f$; then, for all $\bm{x} \in \mathbb{R}^d$,
$
\left|\hat{f}_\delta (\bm{x}) - f(\bm{x})\right|
\leq \delta L_f.
$
\end{lem}
\begin{proof}
From Definition \ref{dfn:newfhat} and Assumption \ref{assum:02}, we have, for all $\bm{x}, \bm{y} \in \mathbb{R}^d$,
\begin{align*}
\left|\hat{f}_\delta(\bm{x}) - f(\bm{x}) \right|
&=\left| \mathbb{E}_{\bm{u}}\left[f(\bm{x}-\delta \bm{u})\right] - f(\bm{x}) \right| \\
&=\left| \mathbb{E}_{\bm{u}} \left[f(\bm{x} - \delta \bm{u}) - f(\bm{x}) \right] \right| \\
&\leq \mathbb{E}_{\bm{u}} \left[\left| f(\bm{x} - \delta \bm{u}) - f(\bm{x}) \right|\right] \\
&\leq \mathbb{E}_{\bm{u}} \left[L_f \| (\bm{x} - \delta \bm{u}) - \bm{x} \|\right] \\
&= \delta L_f \mathbb{E}_{\bm{u}} \left[ \| \bm{u} \| \right] \\
&\leq \delta L_f.
\end{align*}
This completes the proof.
\end{proof}
Here, a larger degree of smoothing should be necessary to make many local optimal solutions of the objective function $f$ disappear and lead the optimizer to the global optimal solution. On the other hand, Lemma \ref{lem:06} implies that the larger the degree of smoothing is, the further away the smoothed function will be from the original function. Therefore, there should be an optimal value for the degree of smoothing that balances the tradeoffs, because if the degree of smoothing is too large, the original function is too damaged and thus cannot be optimized properly, and if it is too small, the function is not smoothed enough and the optimizer falls into a local optimal solution. This knowledge is useful because the degree of smoothing due to stochastic noise in SGD is determined by the learning rate and batch size (see Section \ref{sec:3.2}), so when a constant learning rate and constant batch size are used, the degree of smoothing is constant throughout the training.

The smoothness of the function, and in particular the sharpness of the function around the approximate solution to which the optimizer converged, has been well studied because it has been thought to be related to the generalizability of the model. In this section, we reinforce our theory by experimentally observing the relationship between the degree of smoothing and the sharpness of the function.

Several previous studies \citep{Hochreiter1997Fla, Shirish2017OnL, Izmaliov2018Ave, Li2018Vis, Andriushchenko2023AMo} have addressed the relationship between the sharpness of the function around the approximate solution to which the optimizer converges and the generalization performance of the model. In particular, the hypothesis that flat local solutions have better generalizability than sharp local solutions is at the core of a series of discussions, and several previous studies \citep{Shirish2017OnL, Liang2019Fis, Tsuzuku2020Nor, Petzka2021Rel, Kwon2021ASA} have developed measures of sharpness to confirm this. In this paper, we use ``adaptive sharpness'' \citep{Kwon2021ASA, Andriushchenko2023AMo} as a measure of the sharpness of the function that is invariant to network reparametrization, highly correlated with generalization, and generalizes several existing sharpness definitions. In accordance with \citep{Andriushchenko2023AMo}, let $\mathcal{S}$ be a set of training data; for arbitrary model weights $\bm{w} \in \mathbb{R}^d$, the worst-case adaptive sharpness with radius $\rho \in \mathbb{R}$ and with respect to a vector $\bm{c} \in \mathbb{R}^d$ is defined as
\begin{align*}
S_{\text{max}}^\rho (\bm{w}, \bm{c}) := \mathbb{E}_{\mathcal{S}}\left[ \max_{\| \bm{\delta} \odot \bm{c}^{-1} \|_p \leq \rho} f(\bm{w} + \bm{\delta}) - f(\bm{w}) \right],
\end{align*}
where $\odot /^{-1}$ denotes elementwise multiplication/inversion. Thus, the larger the sharpness value is, the sharper the function around the model weight $\bm{w}$ becomes, with a smaller sharpness leading to higher generalizability.

\begin{figure*}[t]
\begin{minipage}[t]{1\hsize}
\centering
\includegraphics[width=1\linewidth]{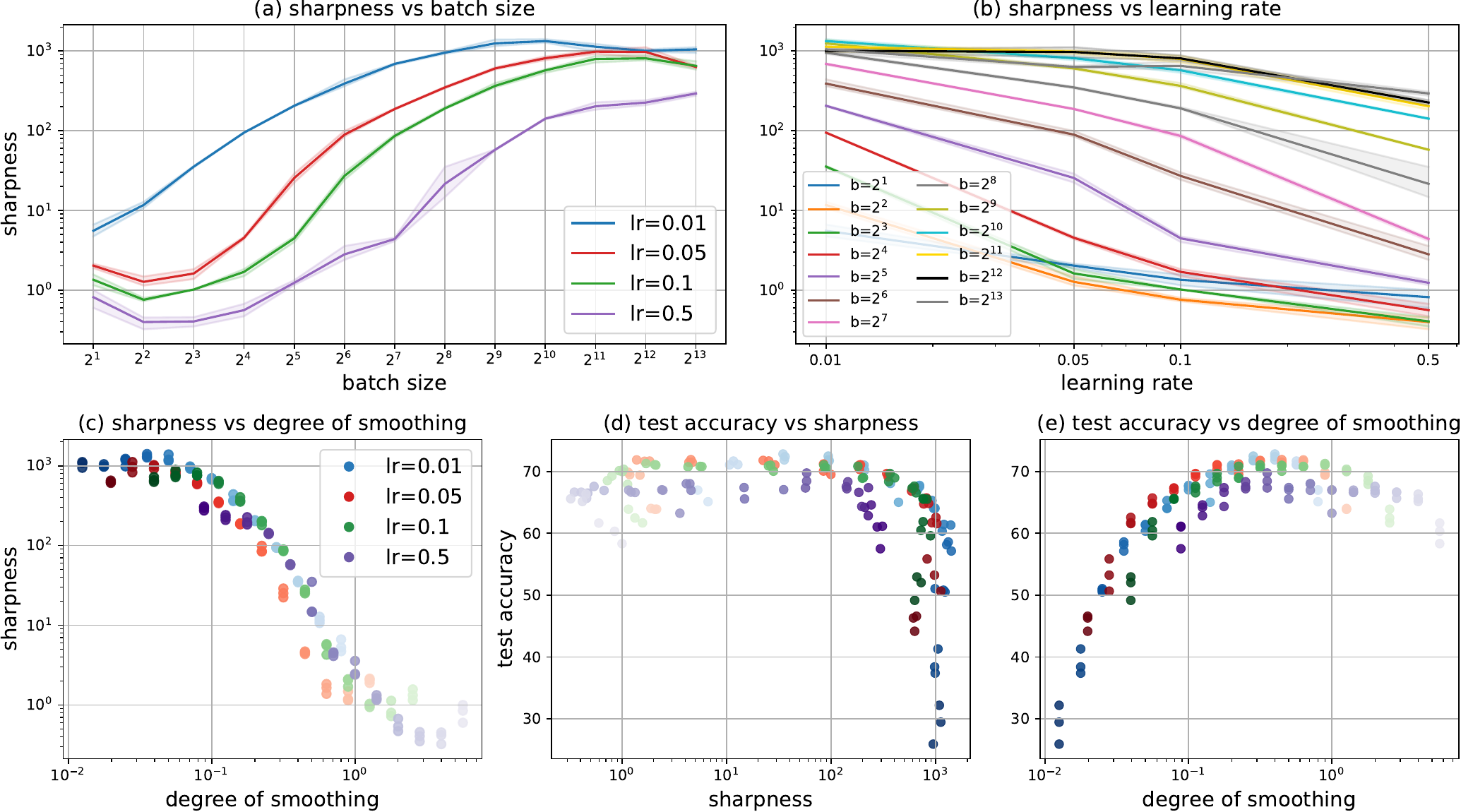}%
\caption{\textbf{(a)} Sharpness around the approximate solution after 200 epochs of ResNet18 training on the CIFAR100 dataset versus batch size used. \textbf{(b)} Sharpness versus learning rate used. \textbf{(c)} Sharpness versus degree of smoothing calculated from learning rate, batch size and estimated variance of the stochastic gradient. \textbf{(d)} Test accuracy after 200 epochs training versus sharpness. \textbf{(e)} Test accuracy versus degree of smoothing. The solid line represents the mean value, and the shaded area represents the maximum and minimum over three runs. The color shade in the scatter plots represents the batch size; the larger the batch size, the darker the color of the plotted points. ``lr'' means learning rate. The experimental results that make up the all graphs are all identical.}
\label{fig:999}
\end{minipage}
\end{figure*}

We trained ResNet18 with the learning rate $\eta \in \{0.01, 0.05, 0.1, 0.1 \}$ and batch size $b \in \{2^1, \ldots, 2^{13}\}$ for 200 epochs on the CIFAR100 dataset and then measured the worst-case $l_{\infty}$ adaptive sharpness of the obtained approximate solution with radius $\rho = 0.0002$ and $\bm{c} = (1, 1, \ldots, 1)^\top \in \mathbb{R}^d$. Our implementation was based on \citep{Andriushchenko2023AMo} and the code used is available on our GitHub. Figure \ref{fig:999} plots the relationship between measured sharpness and the batch size $b$ and the learning rate $\eta$ used for training as well as the degree of smoothing $\delta$ calculated from them. Figure \ref{fig:999} also plots the relationship between test accuracy, sharpness, and degree of smoothing. Three experiments were conducted per combination of learning rate and batch size, with a total of 156 data plots. The variance of the stochastic gradient $C^2$ included in the degree of smoothing $\delta = \eta C/\sqrt{b}$ used values estimated from theory and experiment (see Appendix \ref{sec:CCC} for details).

Figure \ref{fig:999} (a) shows that the larger the batch size is, the larger the sharpness value becomes, whereas (b) shows that the larger the learning rate is, the smaller the sharpness becomes, and (c) shows a greater the degree of smoothing for a smaller sharpness. These experimental results guarantee the our theoretical result that the degree of smoothing $\delta$ is proportional to the learning rate $\eta$ and inversely proportional to the batch size $b$, and they reinforce our theory that the quantity $\eta C/\sqrt{b}$ is the degree of smoothing of the function. Figure \ref{fig:999} (d) also shows that there is no clear correlation between the generalization performance of the model and the sharpness around the approximate solution. This result is also consistent with previous study \citep{Andriushchenko2023AMo}. On the other hand, Figure \ref{fig:999} (e) shows an excellent correlation between generalization performance and the degree of smoothing; generalization performance is clearly a concave function with respect to the degree of smoothing. Thus, a degree of smoothing that is neither too large nor too small leads to high generalization performance. This experimental observation can be supported theoretically (see Lemma \ref{lem:06}). That is, if the degree of smoothing is a constant throughout the training, then there should be an optimal value for the loss function value or test accuracy; for the training of ResNet18 on the CIFAR100 dataset, for example, 0.1 to 1 was the desired value (see Figure \ref{fig:999} (e)). For degrees of smoothing smaller than 0.1, the generalization performance is not good because the function is not sufficiently smoothed so that locally optimal solutions with sharp neighborhoods do not disappear, and the optimizer falls into this trap. On the other hand, a degree of smoothing greater than 1 leads to excessive smoothing and smoothed function becomes too far away from the original function to be properly optimized; the generalization performance is not considered excellent. In addition, the optimal combination of learning rate and batch size that practitioners search for by using grid searches or other methods when training models can be said to be a search for the optimal degree of smoothing. If the optimal degree of smoothing can be better understood, the huge computational cost of the search could be reduced.

\citet{Andriushchenko2023AMo} observed the relationship between sharpness and generalization performance in extensive experiments and found that they were largely uncorrelated, suggesting that the sharpnesss may not be a good indicator of generalization performance and that one should avoid blanket statements like ``flatter minima generalize better''. Figure \ref{fig:999} (d) and (e) show that there is no correlation between sharpness and generalization performance, as in previous study, while there is a correlation between degree of smoothing and generalization performance. Therefore, we can say that degree of smoothing may be a good indicator to theoretically evaluate generalization performance, and it may be too early to say that ``flatter minima generalize better'' is invalid. 

\subsection{Estimation of variance of stochastic gradient}
\label{sec:CCC}
In Section \ref{sec:331}, we need to estimate the variance $C^2$ of the stochastic gradient in order to plot the degree of smoothing $\delta = \eta C/\sqrt{b}$. In general, this is difficult to measure, but several previous studies \citep{Imaizumi2024Ite} have provided the following estimating formula. For some $\epsilon > 0$, when training until $\frac{1}{T}\sum_{k=1}^{T} \mathbb{E}\left[ \| \nabla f(\bm{x}_k) \|^2 \right] \leq \epsilon^2$, the variance of the stochastic gradient can be estimated as 
\begin{align*}
C^2 < \frac{b^\star \epsilon^2}{\eta},
\end{align*}
where $b^\star$ is the batch size that minimizes the amount of computation required for training and $\eta$ is learning rate used in training. We determined the stopping condition $\epsilon$ for each learning rate, measured the batch size that minimized the computational complexity required for the gradient norm of the preceding $t$ steps at time $t$ to average less than $\epsilon$ in training ResNet18 and WideResNet(WRN)-28-10 on the CIFAR100 dataset, and estimated the variance of the stochastic gradient by using an estimation formula (see Table \ref{tab:777}). 

\begin{table}[ht]
\begin{minipage}[t]{0.49\textwidth}
\begin{tabular}{cccc}
 \toprule
 $\eta$ & $\epsilon$ & $b^\star$ & $C^2$ \\
 \midrule
 0.01 & 1.0 & $2^7$ & 12800 \\
 0.05 & 0.5 & $2^9$ & 1280 \\
 0.1 & 0.5 & $2^{10}$ & 1280 \\
 0.5 & 0.5 & $2^{10}$ & 256 \\
 \bottomrule
\end{tabular}
\end{minipage}%
\hspace*{0.02\hsize}
\begin{minipage}[t]{0.49\hsize}
\begin{tabular}{cccc}
 \toprule
 $\eta$ & $\epsilon$ & $b^\star$ & $C^2$ \\
 \midrule
 0.01 & 1.0 & $2^2$ & 400 \\
 0.05 & 0.5 & $2^2$ & 20 \\
 0.1 & 0.5 & $2^{2}$ & 10 \\
 0.5 & 0.5 & $2^{2}$ & 2 \\
 \bottomrule
\end{tabular}
\end{minipage}
\caption{Learning rate $\eta$ and threshold $\epsilon$ used for training, measured optimal batch size $b^\star$ and estimated variance of the stochastic gradient $C^2$ in training ResNet18 (a) and WideResNet-28-10 (b) on the CIFAR100 dataset.}
\label{tab:777}
\end{table}

\end{document}